\documentclass{article}

\usepackage{microtype}
\usepackage{graphicx}
\usepackage{subcaption}
\usepackage{booktabs} 

\usepackage{hyperref}

\usepackage[accepted]{icml2026}

\usepackage{amsmath}
\usepackage{amssymb}
\usepackage{mathtools}
\usepackage{amsthm}

\usepackage[capitalize,noabbrev]{cleveref}

\theoremstyle{plain}
\newtheorem{theorem}{Theorem}[section]

\newtheorem{lemma}[theorem]{Lemma}

\theoremstyle{definition}

\newtheorem{assumption}[theorem]{Assumption}
\theoremstyle{remark}

\usepackage{subcaption}
\usepackage{booktabs}
\usepackage{amssymb}
\usepackage{bbding}
\usepackage{amsmath}
\usepackage{multirow}
\usepackage{xspace}
\newcommand{\name}{\texttt{FOAM}\xspace}
\newcommand{\namec}{\texttt{FOAM-Mini}\xspace}

\usepackage{color, colortbl}
\definecolor{blue1}{rgb}{0.6, 0.6, 0.6}
\definecolor{blue2}{rgb}{0.75, 0.8, 0.99}
\definecolor{blue3}{rgb}{0.84,0.89,0.99}
\definecolor{blue4}{rgb}{0.93,0.95,0.99}

\usepackage[textsize=tiny]{todonotes}

\icmltitlerunning{FOAM: Blocked State Folding for Memory-Efficient LLM Training}

\begin{document}

\twocolumn[
  \icmltitle{FOAM: Blocked State Folding for Memory-Efficient LLM Training}



  \icmlsetsymbol{corr}{*}

  \begin{icmlauthorlist}
    \icmlauthor{Ziqing Wen}{nudt}
    \icmlauthor{Jiahuan Wang}{nudt}
    \icmlauthor{Ping Luo}{nudt}
    \icmlauthor{Dongsheng Li}{nudt}
    \icmlauthor{Tao Sun}{nudt,corr}
  \end{icmlauthorlist}

  \icmlaffiliation{nudt}{National Key Laboratory of Parallel and Distributed Computing, College of Computer Science and Technology, National University of Defense Technology, Changsha, Hunan, China}
  
  \icmlcorrespondingauthor{Tao Sun}{suntao.saltfish@outlook.com}

  \icmlkeywords{Machine Learning, ICML}

  \vskip 0.3in
]



\printAffiliationsAndNotice{}  

\begin{abstract}
Large language models (LLMs) have demonstrated remarkable performance due to their large parameter counts and extensive training data. However, their scale leads to significant memory bottlenecks during training, especially when using memory-intensive optimizers like Adam. Existing memory-efficient approaches often rely on techniques such as singular value decomposition (SVD), projections, or weight freezing, which can introduce substantial computational overhead, require additional memory for projections, or degrade model performance. In this paper, we propose Folded Optimizer with Approximate Moment (\name), a method that compresses optimizer states by computing block-wise gradient means and incorporates a residual correction to recover lost information. Theoretically, \name achieves convergence rates equivalent to vanilla Adam under standard non-convex optimization settings. Empirically, \name eliminates up to 90\% of the memory overhead of optimizer states and accelerates convergence. Furthermore, \name is compatible with other memory-efficient optimizers, delivering performance and throughput that match or surpass both full-rank and existing memory-efficient baselines. Code is available at \url{https://github.com/zqOuO/FOAM}.

\end{abstract}

\section{Introduction}

Large language models (LLMs) have advanced rapidly in recent years and achieved impressive results across a wide range of tasks \cite{Touvron2023Llama2O}. This success is primarily driven by vast training datasets and large model sizes~\cite{zhu2024apollosgdlikememoryadamwlevel_apollo}. As a result, Adam(W)~\cite{Kingma2014AdamAM, loshchilov2017adamw} has become the de facto optimizer for LLM training due to its training efficiency. However, Adam introduces significant memory overhead, consuming twice the model size for storing optimizer states, which makes LLM pre-training and fine-tuning not only compute-intensive but also memory-bound. For instance, even with extremely small training batch sizes, pre-training a 7B model in BF16 still requires at least 58GB of memory \cite{zhao2024galore}. For larger models like GPT-3~\cite{Brown2020LanguageMAgpt3}, the requirement may exceed 700GB.

\begin{figure}[!t]
    \centering
    \includegraphics[width=0.9\linewidth]{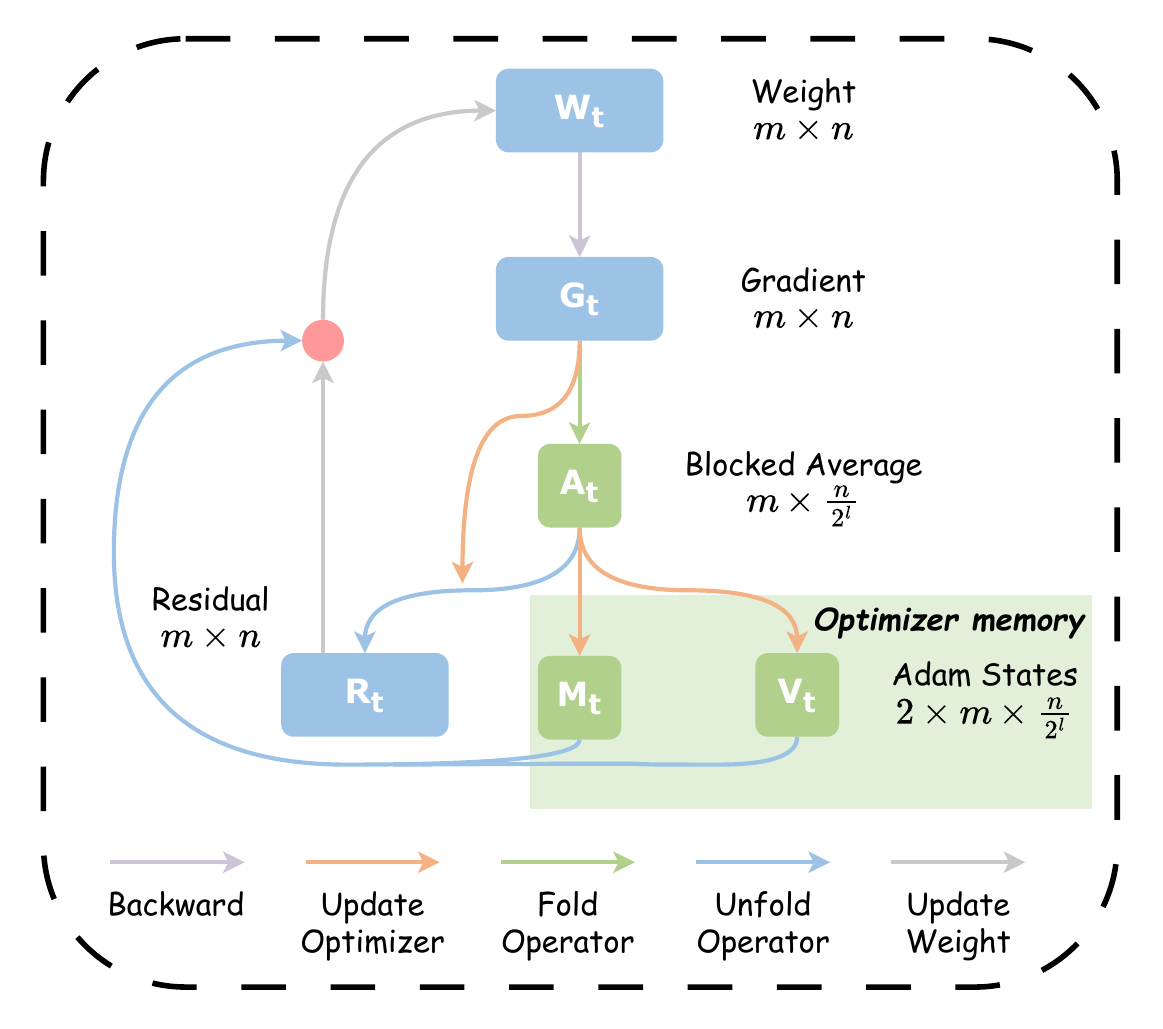}
    \caption{Overview of \name optimizer with a fold level of $l$.}
    \label{fig:placeholder}
\end{figure}

To address Adam’s memory bottleneck, recent work has focused on algorithmic innovations that optimize the optimizer itself rather than relying solely on hardware or system-level improvements. Unlike system-level approaches—such as checkpointing or quantization~\cite{chen2016training,dettmers20218bit,zhang2024qgalore}—which trade off performance or speed for memory savings, algorithmic methods aim to reduce optimizer overhead while preserving full-parameter updates and maintaining Adam’s convergence properties. These approaches can be broadly categorized into three classes: (i) reducing the number of trainable parameters by freezing weights, (ii) applying low-rank projections to compress Adam’s optimizer states, and (iii) sharing learning rates across blocks to eliminate the need for storing per-parameter second-moment estimates. Representative methods include LoRA~\cite{hu2021lora}, GaLore~\cite{zhao2024galore}, and Adam-Mini~\cite{zhang2024adammini}.

\begin{figure*}[!th]
    \centering
    \begin{subfigure}[b]{0.33\linewidth}
        \centering
        \includegraphics[width=\linewidth]{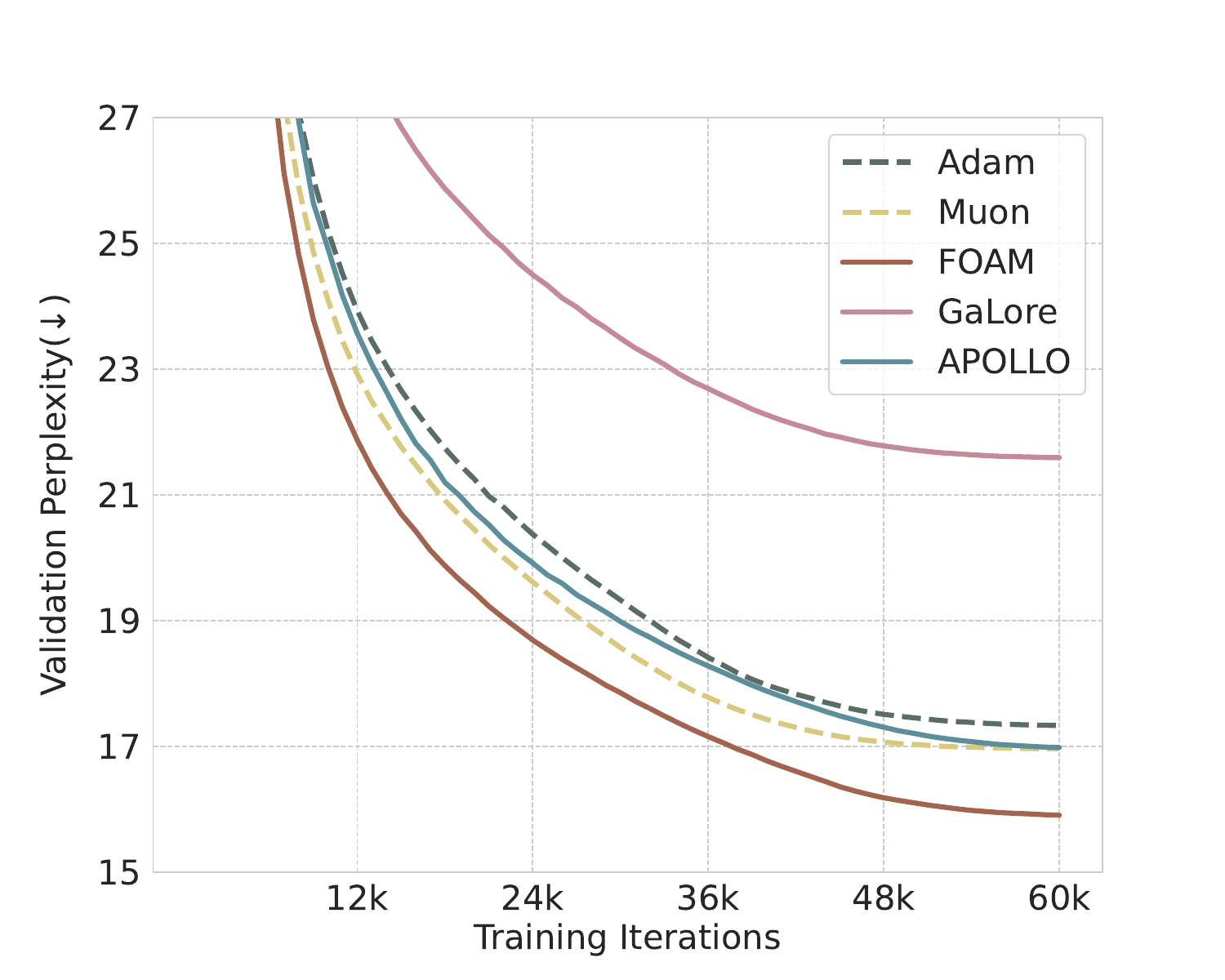}
        \caption{350M model pre-training}
        \label{fig:llama_350m_ppl}
    \end{subfigure}
    \hfill
    \begin{subfigure}[b]{0.33\linewidth}
        \centering
        \includegraphics[width=\linewidth]{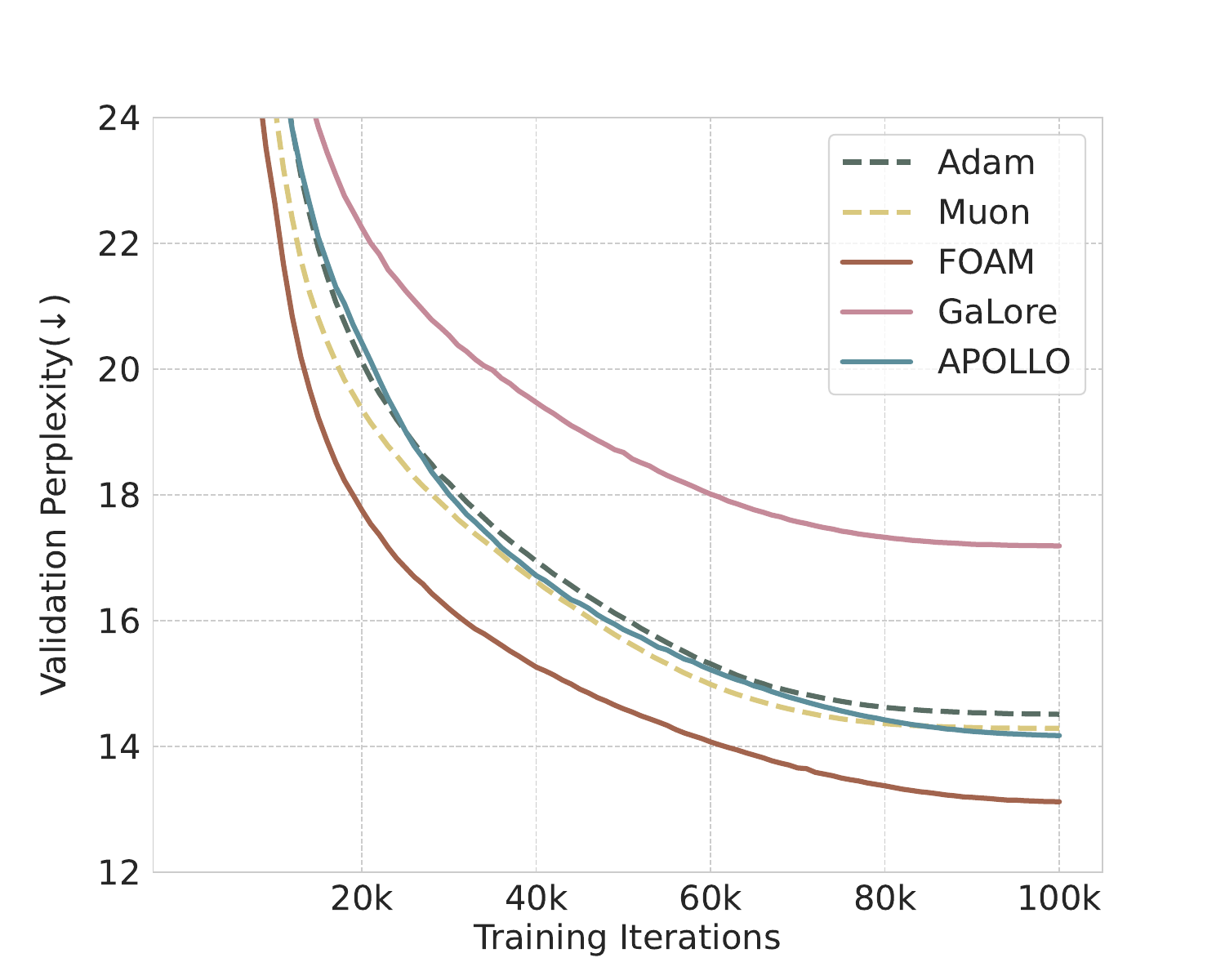}
        \caption{1.3B model pre-training}
        \label{fig:llama_1b_ppl}
    \end{subfigure}
    \hfill
    \begin{subfigure}[b]{0.33\linewidth}
        \centering
        \includegraphics[width=\linewidth, height=0.70\textwidth,trim=0 24 0 20, clip]{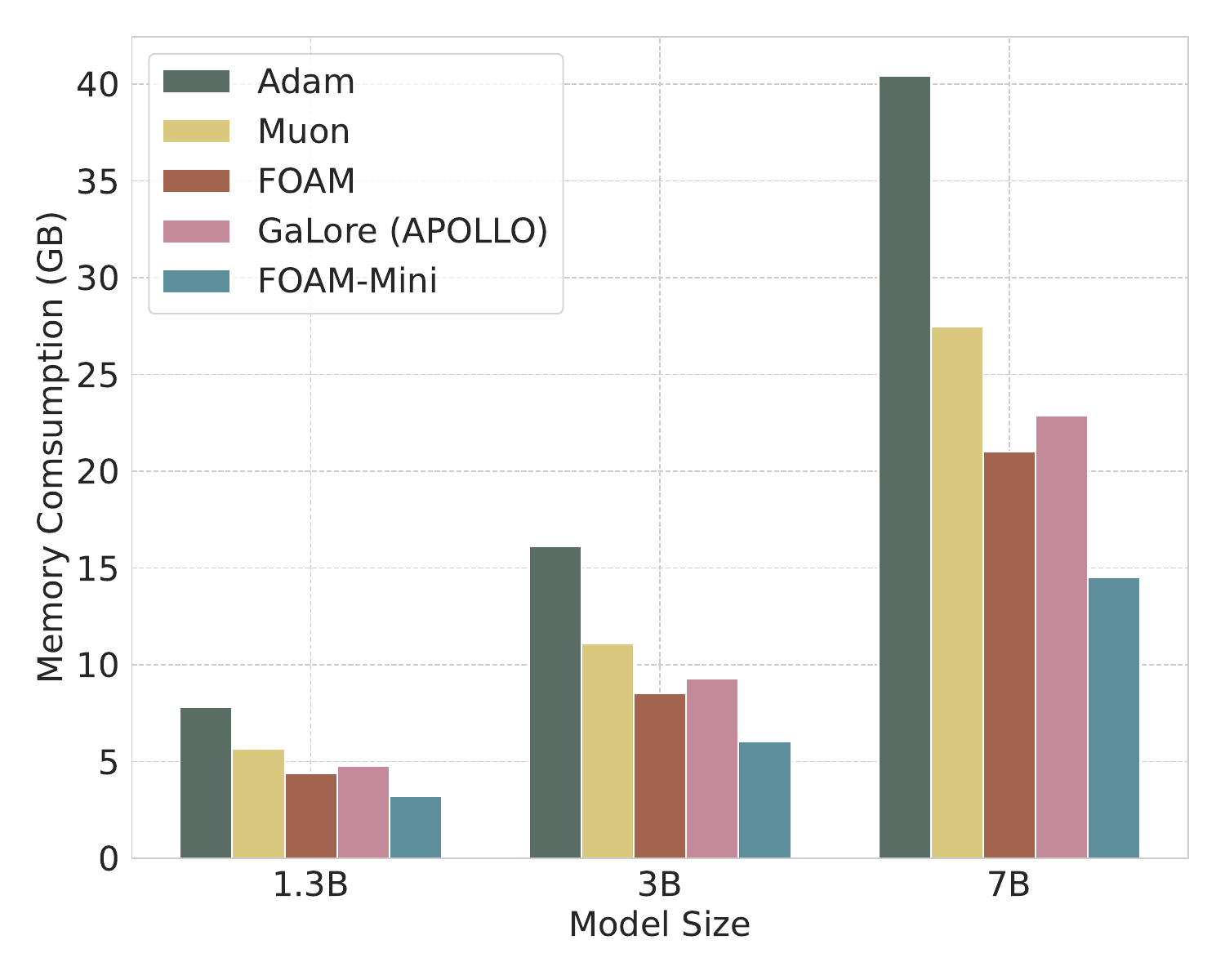}
        \caption{End-to-end memory footprint (BF16)}
        \label{fig:llama_350m_1b_7b_memory}
    \end{subfigure}
    \caption{\textbf{\name performance preview on LLM pre-training.} Figure~\textit{(a)} and ~\textit{(b)}: Perplexity learning curves for pre-training LLaMA-350M and 1.3B on C4. \name demonstrates superior validation perplexity compared with other baselines. Figure~\textit{(c)}: Memory footprint (BF16, model, gradient, and optimizer) for pre-training LLaMA models. \name achieves an approximate 50\% reduction in memory consumption, and \namec further pushes the limit by almost eliminating the memory overhead associated with optimizer states.}
    \label{fig:main}
\end{figure*}

LoRA reduces memory consumption by freezing the original pre-trained weights and learning low-rank decompositions of their updates. This design makes it particularly appealing for fine-tuning scenarios with constrained memory resources. However, by restricting parameter updates to a low-rank subspace, LoRA can limit model expressiveness and occasionally degrade performance on certain tasks~\cite{zhang2023lorafa, xia2024chainlora}.  

GaLore alleviates this limitation by employing singular value decomposition (SVD) to construct projection matrices that compress Adam’s first- and second-moment estimates. While effective, the repeated computation of SVD introduces significant computational overhead, reducing training efficiency. APOLLO~\cite{zhu2024apollosgdlikememoryadamwlevel_apollo} mitigates this by replacing SVD with random projections. GWT~\cite{wen2025breakingmemorylimitsgradient} employs wavelet transforms to compress gradients, bypassing the need for SVD. However, under stringent memory constraints, it still causes a substantial reduction in training throughput and a marked performance decline. Adam-Mini~\cite{zhang2024adammini} reduces memory usage by sharing second-moment estimates across blocks, eliminating the need for per-parameter second-moment storage. This design caps Adam-mini’s achievable optimizer memory efficiency at roughly 50\% but limits further reductions in memory overhead for the first-moment estimates.

To overcome these limitations, we propose \textbf{\name (Folded Optimizer with Approximate Moment)}, a structured state-folding framework that achieves efficient memory compression without sacrificing optimization fidelity. \name compresses both moments of Adam by \emph{blocked average}, which preserves structural information at a coarser level, and incorporates a lightweight \emph{residual correction} to recover the information lost during compression, ensuring that updates remain differentiated across parameters. Unlike previous approaches that rely on weight freezing, low-rank projections, or block-shared learning rates, \name supports full-parameter optimization, eliminates the need for projection matrices, simultaneously compresses Adam’s first- and second-moment estimates, and is highly effective even under stringent memory constraints. Our main contributions are listed as follows:

\begin{itemize}
    \item \textbf{Innovative Optimizer Design}: We propose \name, an efficient memory compression method based on blocked averaging and residual correction. Compared with Full-Adam, \name reduces the end-to-end memory footprint up to 50\% and the memory overhead of the optimizer states by 90\%. Moreover, \name can be applied to other memory-intensive optimizers such as Muon \cite{jordan2024muon} and Adam-mini \cite{zhang2024adammini}.
    \item \textbf{Theoretical Guarantees}: Within the conventional non-convex optimization framework, we prove that \name possesses the same convergence properties as Adam, theoretically ensuring its effectiveness in modern deep learning training tasks. 
    \item \textbf{Extensive Evaluation}: We evaluate \name on LLM pre-training (LLaMA, Qwen, GPT, DeBERTa models) and fine-tuning (GLUE, MMLU benchmarks) across model sizes (60M–7B) and sequence lengths (256–2048). \name outperforms memory-efficient baselines in perplexity and convergence speed without incurring a large computational budget. This makes \name\ a practical, highly efficient, and optimizer-agnostic solution for alleviating the memory bottlenecks in large-scale language model training.
\end{itemize}

\section{Related Works}

\begin{table*}[!th]
    \centering
    \caption{\textbf{Comparison of memory and computational complexity across methods.} Let $W \in \mathbb{R}^{m\times n}$ with $m \leq n$. GaLore and APOLLO use a rank $r$, whereas \name uses a level $l$.}
    \setlength{\tabcolsep}{10pt}
    \resizebox{\linewidth}{!}{
    {\begin{tabular}{l|ccccccc}
    \toprule
      & Adam   &Muon & Adam-Mini & GaLore & APOLLO & \cellcolor{blue4}\name & \cellcolor{blue2}\namec \\
      \midrule
      Optimizer Memory  & $2mn$ & $mn$ & $mn$ & $mr+2nr$ & $mr+2nr$ & \cellcolor{blue4}$mn/2^{l-1}$ & \cellcolor{blue2}$2m$\\   
      Complexity & $\mathcal{O}(mn)$ & $\mathcal{O}(m^2n)$ & $\mathcal{O}(mn)$ & $\mathcal{O}(m^2n)$ & $\mathcal{O}(mnr)$ & \cellcolor{blue4}$\mathcal{O}(mn)$ & \cellcolor{blue2}$\mathcal{O}(mn)$\\
    \bottomrule
    \end{tabular}}
    }
    \label{tab:memory_comp}
    
\end{table*}

Recent research has explored various algorithmic strategies to reduce the memory overhead of LLM training. Our work follows this line of inquiry, aiming to enhance memory efficiency through optimizer and state-level innovations.  

A prominent direction focuses on reducing the number of trainable parameters. LoRA~\cite{hu2021lora} achieves this by freezing the pre-trained model weights and learning low-rank decompositions of their updates, substantially lowering memory usage during fine-tuning. The approach has inspired numerous extensions and refinements, including Batched LoRA~\cite{Wen2023BatchedLorA}, ReLoRA~\cite{Lialin2023ReLoRAHT}, FLORA~\cite{hao2024flora}, DoRA~\cite{liu2024dora}, and BlockedLLM~\cite{ramesh2024blockllm}.  

Another major line of work targets the substantial memory footprint of optimizer states, particularly in Adam. These methods project gradients and perform updates within a low-dimensional subspace, thereby reducing the need to store high-dimensional moment estimates. Representative examples include Adafactor~\cite{Shazeer2018AdafactorAL}, GaLore~\cite{zhao2024galore}, Fira~\cite{Chen2024FiraCW}, GWT~\cite{wen2025breakingmemorylimitsgradient}, APOLLO~\cite{zhu2024apollosgdlikememoryadamwlevel_apollo}, LDAdam~\cite{robert2025ldadam}, and SubTrack++~\cite{rajabi2025subtrack++}. Both Fira, GWT, APOLLO, LDAdam, and SubTrack++ have demonstrated competitive or superior performance compared to full-rank Adam, underscoring the effectiveness of subspace-based optimization.  

Other approaches, such as Adam-Mini~\cite{zhang2024adammini} and SGDSal~\cite{xu2024adamlearningratescalingsgdsal}, reduce memory consumption by employing learning rate sharing schemes. AdamS~\cite{zhang2025adams} leverages momentum to introduce a normalizer. Meanwhile, Muon~\cite{jordan2024muon, liu2025muon} and Scion~\cite{pethick2025trainingscion} incorporate orthogonalized gradients, enabling SGD-style updates.


\section{Motivation and Algorithm}
In this section, we first present the detailed procedure for the Adam(W) \cite{Kingma2014AdamAM, loshchilov2017adamw}, followed by introducing our \textbf{Folded Optimizer with Approximate Moments (\name)}.

\subsection{Full-Adam} 
At each time step $t$, the Adam optimizer updates the weight matrix $W \in \mathbb{R}^{m \times n}$ by leveraging the first-order moment and the second-order moment by  
\begin{equation} \label{eq:adam_updates} 
W_{t+1} = W_t - \eta_t \cdot \frac{M_{t}}{\sqrt{V_t + \epsilon}} , 
\end{equation} 
where $\eta_{t} > 0$ denotes the step size (i.e., the learning rate) at time step $t$, and $\cdot$ denotes element-wise multiplication. Here, $M_t,V_t \in \mathbb{R}^{m \times n},$ denote the first and second-order moment, and $\epsilon > 0$ for numerical stability. The moments $M_t$ and $V_t$ are updated by
\begin{equation} \label{eq:adam_moment_update} 
\begin{aligned}
    &M_t = \beta_1 \cdot M_{t-1} + (1 - \beta_1) \cdot G_t, \\ &V_t = \beta_2 \cdot V_{t-1} + (1 - \beta_2) \cdot G_t^{2},
\end{aligned} 
\end{equation} 
where $G_t\in \mathbb{R}^{m\times n}$ denotes the stochastic gradient at step $t$, and $\beta_1, \beta_2 \in [0, 1)$ are the decay rates. Adam stores both $M_{t}$ and $V_{t}$ to implement its adaptive update scheme. This results in an additional memory overhead of $2 \times m \times n$ elements for the optimizer states.

\subsection{\name}
The strong performance of the Adam optimizer stems from
the first- and second-moment estimates, but this comes at the cost of twice the memory overhead. To address this challenge, we propose \name, which compresses both moments by partitioning adjacent gradient entries into blocks and sharing optimizer states within each block. In addition, \name introduces a residual correction mechanism to recover information lost due to compression. Specifically, given a gradient $G_{t}\in \mathbb{R}^{m\times n}$ at time step $t$, and fold-level $l$, \name replaces each group of $2^l$ consecutive elements with their mean value for storage (we assume $n\bmod 2^l=0$, which can be ensured via padding). Specifically, suppose the fold operator matrix $A^{(l)} \in \mathbb{R}^{n\times \frac{n}{2^l}}$ satisfies:
\begin{equation}\label{eq:define_A}
    {A^{(l)}_{i,j}}= \begin{cases}
        \frac{1}{2^{l}}, & \text{if}\ (j-1) \cdot 2^{l}+1\leq i \leq j \cdot 2^{l} \\
        0, & \text{others}
    \end{cases}
\end{equation}
we can obtain the compressed gradient $\tilde{G}_{t}=G_{t}A^{(l)} \in \mathbb{R}^{m\times \frac{n}{2^l}}$, and track the first- and second- moment by
\begin{equation}
    \begin{aligned}
        \label{eq:m_v_update_res}
        \tilde{M}_{t} &= \beta_1 \cdot \tilde{M}_{t-1} + (1-\beta_1)\cdot \tilde{G}_{t} \\
        \tilde{V}_{t} &= \beta_2 \cdot \tilde{V}_{t-1} + (1-\beta_2) \cdot \tilde{G}_{t}^2
    \end{aligned}
\end{equation}

To apply the compressed update to the full parameter, we expand it to the original dimension by an \emph{unfold} operator. Specifically, we define the unfold matrix $E^{(l)} \in \mathbb{R}^{\frac{n}{2^{l}} \times n}$ as
\begin{equation}\label{eq:define_E}
    E^{(l)}_{i,j} =
    \begin{cases}
        1, & \text{if } (i-1) \cdot 2^{l} + 1 \le j \le i \cdot 2^{l},\\[4pt]
        0, & \text{otherwise}.
    \end{cases}
\end{equation}
This operator replicates each compressed entry across its corresponding block, restoring the update to the full parameter dimension. However, directly replicating the update across dimensions causes all elements within the same block to share identical update values, resulting in weight differences that rely solely on initialization. To overcome this, we introduce the residual $R_{t}$, which is defined as the fold–unfold operator error at time step $t$
\begin{equation}
    \nonumber
    R_{t} = G_{t} - \tilde{G}_{t}E^{(l)} = G_{t} - G_{t} A^{(l)}E^{(l)}
\end{equation}
The residual depends solely on the gradient at time step $t$; once computed, it can be released and does not need to be maintained in memory. We introduce the residual term into the expanded estimates of the first-order moment as follows
\begin{equation}
    \label{eq:foam_m_v_res}
    M_{t} = \tilde{M}_{t} E^{(l)} + R_{t},\quad V_{t} = \tilde{V_{t}}E^{(l)} + R_{t}^{2},
\end{equation}
and update the parameters by Eq.~\eqref{eq:adam_updates}. The injection of residual to $M_{t}$ prevents identical parameter updates within the same block due to broadcasting operations, and ensures that the first-moment estimate $M_{t}$ includes the complete gradient information $G_{t}$ of the current step. Similar to Adam-mini, \name allows parameters within the same block to share the same second-order moment estimate. The key difference is that \name uses a finer-grained block structure and adds an additional residual term to the unfolded second-order moment. We will reveal in Appendix Figure~\ref{fig:ablation_3b_res} that adding residuals to the second-order moment stabilizes training.

Evidently, \name tracks only ${1}/{2^{l-1}}$ optimizer state of Full-Adam. We present the pseudocode for \name in Algorithm~\ref{algo:fold_adam_algo}. We apply \name to the MLP and Attention modules, and Adam to other modules such as Embedding. A scaling coefficient $\alpha$ is introduced to control the ratio of learning rates between \name modules and Adam modules. This design is commonly adopted in memory-efficient
optimizers~\cite{zhao2024galore,jordan2024muon,zhu2024apollosgdlikememoryadamwlevel_apollo}.
\begin{algorithm}[!th]
   \caption{\name optimizer}\label{algo:fold_adam_algo}
    \begin{algorithmic}
    \def\arraystretch{1.0} 
   \STATE {\bfseries Input:} 2-D Weight matrix $ {W}$, step size $\{\eta\}_{t\ge1}$, batch size $b$, decay rates $\beta_{1}, \beta_{2}$, iteration $T$, $\epsilon$ for numerical stability, scaling coefficient  $\alpha$, fold level $l$, fold and unfold operation $A^{(l)},E^{(l)}$, and $t=1$.
   \REPEAT
   \STATE $ {G}_{t} \leftarrow 
   \frac{1}{b}\sum_{i=1}^{b}\nabla_{ {W}}f_{i}( {W}_{t}) $ \hfill \COMMENT{Batch gradient}
   \STATE $ \tilde{G}_{t}\leftarrow  {G}_{t} A^{(l)}$ \hfill \COMMENT{Compress the gradient}
   \STATE $R_{t} \gets {G}_{t} - \tilde{G}_{t}E^{(l)}$ \hfill \COMMENT{Compute the residual}
    \IF{$t=1$}
    \STATE Initialize $ \tilde{M}_{0}, \tilde{V}_{0} \gets 0$
    \ENDIF
    \\
    \hrulefill
   \STATE \textbf{Adam states update }
    \STATE \hspace{2.5mm} $ \tilde{M}_{t} \leftarrow \beta_{1} \cdot  \tilde{M}_{t-1} + (1 - \beta_{1}) \cdot  \tilde{G}_{t}$ 
    \STATE \hspace{2.5mm} $ \tilde{V}_{t} \gets \beta_2 \cdot  \tilde{V}_{t-1} + (1 - \beta_2) \cdot  \tilde{G}_{t}^{2}$ 
    \STATE \hspace{2.5mm} ${M_{t}} \gets \tilde{M}_{t}E^{(l)} + R_{t}$
    \STATE \hspace{2.5mm} ${V_{t}} \gets \tilde{V}_{t}E^{(l)} + R_{t}^{2}$ \hfill \COMMENT{Unfold the states}
    \\
   \hrulefill
   \STATE  $ {W}_{t} \gets  {W}_{t-1} -  \eta_t \cdot \alpha \cdot \frac{M_{t}}{\sqrt{V_{t}}+\epsilon}$ \hfill \COMMENT{Update weights}
   \STATE $t\gets t+1$
   \UNTIL{$t=T$}
   \STATE \textbf{return} $ {W}_{t}$
\end{algorithmic}

\end{algorithm}

\begin{table*}[!t]
    \centering
    \caption{\textbf{Final validation perplexity (lower is better, 3 independent runs for 60M to 350M models) and estimated memory usage (BF16, model, gradient, and optimizer states) for pre-training LLaMA models on the C4 dataset.} 
    \name consistently outperforms baselines by achieving comparable or lower validation PPL while requiring significantly less memory.}

    \label{tab:validation_llama_60m-1b}
    \resizebox{\linewidth}{!}{
    \begin{tabular}{l|cc|cc|cc|cc}
    \toprule
        \multirow{2}{*}{Methods} & \multicolumn{2}{c|}{\textbf{60M}} & \multicolumn{2}{c|}{\textbf{130M}} & \multicolumn{2}{c|}{\textbf{350M}} & \multicolumn{2}{c}{\textbf{1.3B}} \\
        & Perplexity & Memory & Perplexity & Memory & Perplexity & Memory & Perplexity & Memory \\
        \midrule
        Full-Adam & $29.57_{\pm 0.02}$ & 0.34G & $22.86_{\pm 0.00}$ & 0.80G & $17.33_{\pm 0.00}$ & 2.20G & $14.51_{\pm 0.00}$ & 8.03G\\
        Muon & $28.93_{\pm 0.03}$ & 0.30G & $22.75_{\pm 0.01}$ & 0.63G & $16.96_{\pm 0.01}$ & 1.60G & $14.28_{\pm 0.00}$ & 5.61G \\
        Adam-Mini & $29.63_{\pm 0.05}$ & 0.22G & $23.73_{\pm 0.01}$ & 0.53G & $17.83_{\pm 0.01}$ & 1.46G & $15.10_{\pm 0.00}$ & 5.35G \\
        LDAdamW & $29.27_{\pm 0.07}$ & 0.28G & $22.86_{\pm 0.02}$ & 0.57G & $17.53_{\pm 0.01}$& 1.38G & $14.88_{\pm 0.00}$ & 4.76G \\
        \midrule
        GaLore-1/4 & $34.38_{\pm 0.03}$ & 0.28G & $26.47_{\pm 0.01}$ & 0.57G & $19.36_{\pm 0.02}$ & 1.38G & $15.66_{\pm 0.00}$ & 4.76G\\
        APOLLO-1/4 &$ 31.18_{\pm 0.09}$ & 0.28G & $23.35_{\pm 0.21}$ & 0.57G & $16.73_{\pm 0.05}$ & 1.38G & $14.20_{\pm 0.00}$ & 4.76G\\
        \rowcolor{blue4}
         \name-2 &  $\textbf{28.53}_{\pm 0.01}$ &  0.27G &   $\textbf{22.51}_{\pm 0.00}$ &  0.54G & $ \textbf{15.87}_{\pm 0.00}$ &  1.30G &  $\textbf{13.13}_{\pm 0.00}$ &  4.45G \\
        \midrule
        GaLore-1/8 & $39.94_{\pm 1.24}$ & 0.26G & $30.02_{\pm 0.39}$ & 0.52G & $21.59_{\pm 0.23}$ & 1.23G & $17.52_{\pm 0.00}$ & 4.15G \\
        APOLLO-1/8 & $31.53_{\pm 0.06}$ & 0.26G & $23.74_{\pm 0.17}$ & 0.52G & $16.98_{\pm 0.10}$ & 1.23G & $14.32_{\pm 0.00}$ & 4.15G\\
        \rowcolor{blue4}
        \name-3 & $\textbf{28.79}_{\pm 0.05}$ & 0.25G & $\textbf{22.58}_{\pm 0.03}$ & 0.50G & $\textbf{15.94}_{\pm 0.02}$ & 1.14G & $\textbf{13.19}_{\pm 0.00}$ & 3.97G \\
        \midrule
        APOLLO-Mini & $31.58_{\pm 0.09}$ & 0.24G & $23.83_{\pm 0.07}$ & 0.46G & $17.17_{\pm 0.06}$  & 1.00G & $14.18_{\pm 0.00}$ & 3.20G\\
        GWT-Mini & $32.94_{\pm 0.06}$ & 0.24G & $23.84_{\pm 0.05}$ & 0.46G & $18.12_{\pm 0.03}$ & 1.00G & $14.99_{\pm 0.00}$ & 3.20G \\
        \rowcolor{blue2}
        \namec & $\textbf{29.71}_{\pm 0.04}$ & 0.24G & $\textbf{23.10}_{\pm 0.05}$ & 0.46G & $\textbf{16.53}_{\pm 0.03}$ & 1.00G & $\textbf{13.43}_{\pm 0.00}$ & 3.20G\\
        \midrule
        Training Tokens & \multicolumn{2}{|c}{1.3B} & \multicolumn{2}{|c}{2.6B} & \multicolumn{2}{|c}{7.8B} & \multicolumn{2}{|c}{13.1B} \\
         \bottomrule
    \end{tabular}
    }

\end{table*} 

\section{Theoretical Results}
In this section, we provide the theoretical convergence of \name, demonstrating that under the traditional non-convex optimization assumptions, \name, even with compressed optimizer states, achieves the same convergence rate as Full-Adam. Before beginning our theoretical analysis, we first collect the following assumptions:
\begin{assumption}[$L$-smoothness]\label{ass:lipschitz}
    Assume the gradient of the loss function $f$ is $L$-smoothness, i.e., 
    \begin{equation}
        \nonumber
        \|\nabla f(W) - \nabla f(W')\| \leq L\|W-W'\|, \forall\ W,W' \in \mathbb{R}^{m \times n}.
    \end{equation}
\end{assumption}

\begin{assumption}[Bounded variance of the stochastic gradient]\label{ass:unbiased}
    Assume that the gradient follows the following noise decomposition
    \begin{equation}
        \nonumber
        G_{t} = \nabla f(W_{t}) + \xi_{t},\quad \mathbb{E}\left[\xi_{t}\mid W_{t}\right] = 0,\quad \mathbb{E}\left[\|\xi_{t}\|^{2}\right] \leq \sigma^2.
    \end{equation}
\end{assumption}

\begin{assumption}[Bounded gradient]\label{ass:bound_gradient}
    There exists a constant $C \geq 0$ such that
    $$
    \|\nabla f(W_{t})\| \leq C, \quad \forall\ t.
    $$
\end{assumption}

These assumptions are standard in stochastic optimization and provide the foundation for convergence analysis.

\begin{theorem}[Convergence of \name]\label{theo:convergence}
    Let $\{W_{t}\}_{t\geq 1}$ be generated by Algorithm~\ref{algo:fold_adam_algo}, under the assumptions of Assumptions~\ref{ass:lipschitz} - \ref{ass:bound_gradient}, and with $\eta_{t} = \eta_0 / \sqrt{t}$. Then, we have
    \small{
    \begin{equation}\nonumber
        \min_{1\leq t\leq T}\mathbb{E}\left[\|\nabla f(W_{t})\|^{2}\right] = \mathcal{O}\left( \frac{\log T + \delta_l^2}{\sqrt{T}} \right) + \mathcal{O}\left( \frac{\sigma^2 \log T}{\sqrt{T}} \right).
    \end{equation}}
\end{theorem}
where $\delta_l = \max_{1 \leq t \leq T} \frac{\|R_{t}\|}{\|G_t\|}$ denotes the energy ratio of the residual and satisfies $\delta_l \leq 1,\forall\ l$. 

This result demonstrates that \name achieves the same convergence rate of $\mathcal{O}(\frac{\log T}{\sqrt{T}})$ as Adam~\cite{reddi2019convergence_adam, zhou2024convergence_adam}. Furthermore, $\delta_l \leq 1$ for all $l$ does not affect the primary convergence rate, thus theoretically validating the effectiveness of \name. In the Appendix Figure~\ref{fig:residual_energy}, we provide the variation of $\frac{\|R_{t}\|}{\|G_t\|}$ throughout training for different of $l$, where a higher $l$ corresponds to a larger $\delta_{l}$.

\section{Experimental Results}
In this section, we experimentally validate the effectiveness of our proposed method, focusing on both pre-training and fine-tuning tasks. For pre-training, we train LLaMA~\cite{Touvron2023LLaMAOA} models of various sizes on the English portion of the C4~\cite{2019c4} dataset. For fine-tuning, we evaluate a range of open-source models on standard downstream tasks. Detailed descriptions of experimental settings and computational environments can be found in the Appendix. Throughout our experiments, we use \name-2 to denote the Adam optimizer with $l=2$. We also define \namec as a variant of \name that applies $l=\lfloor\log_2{\text{dim}_{\text{hidden}}}\rfloor$, resulting in memory consumption approximately equivalent to that of rank-1 methods.

\subsection{Memory-Efficient Pre-Training} 

We demonstrate that our proposed method, \name, is highly effective for the pre-training of LLMs, achieving superior performance in terms of validation perplexity (PPL), training throughput, memory efficiency, and convergence speed across a range of LLaMA model sizes. Notably, \name-c—a variant with an extremely low memory footprint for optimizer states—performs on par with both full-rank baselines and existing memory-efficient methods.

\textbf{Training Configuration.}
All pre-training experiments are conducted using the BF16 format to reduce memory consumption. We adopt the LLaMA model configurations provided by~\citet{zhao2024galore}. Following prior work, we use a maximum sequence length of 256 and a batch size of 512, yielding a total of 131K tokens per batch. The learning rate is linearly warmed up during the first 10\% of training steps and decayed thereafter using a cosine schedule.

\begin{table*}[!t]
    \centering
    \caption{\textbf{Pre-training LLaMA-3B and LLaMA-7B models on the C4 dataset.} 
    We report the validation PPL across training steps, wall-clock training time, and real memory consumption with gradient checkpointing during training. We train LLaMA-3B on 16 NVIDIA RTX 3090 GPUs (24GB each), and LLaMA-7B on 4 NVIDIA H100 GPUs (80GB each).}
    \label{tab:pre_train_llama_3b}
    \resizebox{0.9\linewidth}{!}{
    \begin{tabular}{l|l|c|ccccc|cc}
        \toprule
        Models & Methods &\textbf{Memory}  & \textbf{30k} & \textbf{60k} & \textbf{90k} & \textbf{120k} & \textbf{150k} & \textbf{Time (h)} & \textbf{Tokens/s}\\
        \midrule
        \multirow{5}{*}{LLaMA-3B} & 8-bit Adam & 5.30G & 18.63 & 15.69 & 14.36 & 14.19 & - & 130.1 & 34.5K \\
        & Muon & 5.47G & 18.14 & 15.40 & 14.13 & 14.02 & -  & 297.3 & 16.3K \\
        & GaLore-1/4 & 3.91G & 18.44 & 15.98 & 14.90 & 14.73 & -& 143.6 & 35.7K \\
        & APOLLO-1/4 & 3.91G & 19.49 & 15.40 & 14.09 & 13.75 & -& 125.6 & 35.9K \\
        & \cellcolor{blue4}\name-2 & \cellcolor{blue4}3.20G&  \cellcolor{blue4}15.71 & \cellcolor{blue4}13.46 & \cellcolor{blue4}12.29 & \cellcolor{blue4}\textbf{11.98} &\cellcolor{blue4}-& \cellcolor{blue4}126.2 & \cellcolor{blue4}35.7K\\
        \midrule
        \multirow{5}{*}{LLaMA-7B} & 8-bit Adam & 13.5G & 18.73 & 16.13 & 14.56 & 13.40 & 13.23  & 285.5 & 19.9K \\
        & Muon & 14.0G & 18.16 & 15.71 & 13.87 & 13.23 & 13.19 & 350.0 & 16.0K \\
        & GaLore-1/4 & 9.40G & 18.35 & 15.63 & 14.33 & 13.77 & 13.69 & 341.8 & 20.9K \\
        & APOLLO-1/4 & 9.40G & 18.49 & 15.17 & 13.54 & 12.80 & 12.63 & 270.1 & 21.1K \\
        & \cellcolor{blue4}\name-2 &\cellcolor{blue4}7.52G & \cellcolor{blue4}15.78 & \cellcolor{blue4}13.33 & \cellcolor{blue4}12.01 & \cellcolor{blue4}11.39 & \cellcolor{blue4}\textbf{11.13} & \cellcolor{blue4}270.2 & \cellcolor{blue4}21.0K \\
        \bottomrule
    \end{tabular}
    }
\end{table*}

\textbf{Baselines.}
To ensure a comprehensive comparison, we reproduce results for several representative optimizers. These include: \textbf{Full-Adam}~\cite{Kingma2014AdamAM}, the standard optimizer for training large models.
\textbf{Muon}~\cite{liu2025muon}, an SGD-momentum variant with Newton-Schulz (N-S) iterations;
\textbf{Adam-Mini}~\cite{zhang2024adammini}, Adam with block-shared second-moment estimate; 
\textbf{GaLore}~\cite{zhao2024galore}, a memory-efficient Adam with low-rank gradient projections; 
\textbf{LDAdamW}~\cite{robert2025ldadam}, a GaLore variant with continually rotating and persistent error feedback; 
\textbf{APOLLO}~\cite{zhu2024apollosgdlikememoryadamwlevel_apollo}, which reduces memory usage via random projections. \textbf{GWT}~\cite{wen2025breakingmemorylimitsgradient}, which reduces memory by wavelet transforms.

For Full-Adam, Muon, and Adam-Mini optimizers, we report the best results from a learning rate sweep over $\{1\mathrm{e}{-4}, 2.5\mathrm{e}{-4},\ 5\mathrm{e}{-4},\ 1\mathrm{e}{-3},\ 2.5\mathrm{e}{-3},\ 5\mathrm{e}{-3},\ 1\mathrm{e}{-2}\}$. For GaLore, APOLLO, GWT, \name and \namec, we report the best results from a learning rate sweep over $\{1\mathrm{e}{-3},\ 2.5\mathrm{e}{-3},\ 5\mathrm{e}{-3},\ 1\mathrm{e}{-2}\}$ and choose the $\alpha$ in $\{0.25, 0.5, 1.0\}$. In particular, GaLore-1/8 refers to the configuration where the projection rank is one-eighth of the hidden dimension—corresponding to the memory footprint of \name with 3-level gradient folding.

\textbf{Main Reults.}
We evaluate the performance of \name and \namec by comparing their memory overhead (calculated layer-by-layer, with details in the Appendix) and final validation PPL against various baseline optimizers. The final PPL results for LLaMA models ranging from 60M to 1.3B parameters are summarized in Table~\ref{tab:validation_llama_60m-1b}. Results for pre-training LLaMA-3B and LLaMA-7B with gradient checkpointing~\cite{chen2016training} are presented in Table~\ref{tab:pre_train_llama_3b}. Additionally, we include the PPL learning curves for LLaMA-350M and 1.3B in Figures~\ref{fig:llama_350m_ppl} and~\ref{fig:llama_1b_ppl}, for LLaMA-60M and 130M in Figure~\ref{fig:ppl_curve_60m_130m}, and the real training memory consumption for LLaMA-1.3B in Appendix Table~\ref{tab:throught_1B}. These results lead to several key observations:

\textbf{\name Demonstrates Superior Effectiveness in Pre-training:}
Under varying memory constraints (i.e., $ d_{\text{model}}/4$, $d_{\text{model}}/8$, and rank 1), \name consistently outperforms all other memory-efficient optimizers in final validation PPL, even surpassing Full-Adam and Muon with tuned learning rates, while maintaining significantly lower memory usage. Specifically, for the LLaMA-1.3B model, \name-3 consumes approximately 3.97 GB of memory, reducing optimizer memory overhead by 79\% and total memory overhead by 50\% compared to Adam. It also achieves a 20\% reduction in memory usage relative to GaLore and APOLLO, both configured at $1/4, d_{\text{model}}$ in the 7B model. Additionally, \namec decreases optimizer memory overhead by 90\%, while maintaining performance comparable to Adam. \textbf{Faster Convergence and Higher Throughput:}
As shown in Figures~\ref{fig:llama_350m_ppl} and~\ref{fig:llama_1b_ppl}, \name achieves a 2× speedup in convergence steps, outperforming or matching other memory-efficient optimizers in terms of PPL. For larger models, such as LLaMA-3B and LLaMA-7B (as presented in Table~\ref{tab:pre_train_llama_3b}), \name continues to outperform the tested baselines in both convergence speed and final validation PPL, while matching the training throughput of SVD-free APOLLO.

\subsection{Memory-Efficient Fine-Tuning}

In this section, we further evaluate \name in LLM fine-tuning, a widely adopted practice in both academia and industry. Our experiments include fine-tuning RoBERTa-Large~\cite{liu2019roberta} on the GLUE benchmark~\cite{Wang2018GLUEAM} and adapting several open-source models for evaluation on the MMLU benchmark~\cite{hendrycks2020measuringmmlu}.

\textbf{Training Configuration.}
For GLUE fine-tuning, we fine-tune each task for 3 epochs. For MMLU tasks, we employ 3 open-source models: Gemma3-1B~\cite{team2025gemma3}, LLaMA3.2-3B~\cite{grattafiori2024llama3}, and Qwen2.5-7B~\cite{yang2024qwen2.5} with a sequence length of 2048. Detailed experimental settings are provided in the Appendix.

\begin{table}[!ht]
    \centering
    \caption{\textbf{Evaluating \name for memory-efficient fine-tuning on the MMLU benchmark} (1 A100 40GB GPU). 
    We report the best average accuracy obtained by sweeping the learning rate over the range $\{1\mathrm{e}{-5},\ 2.5\mathrm{e}{-5},\ 5\mathrm{e}{-5},\ 1\mathrm{e}{-4},\ 1.5\mathrm{e}{-4},\ 2\mathrm{e}{-4}\}$.
    }
    \label{tab:fine_tune_mmlu}
    \resizebox{0.98\linewidth}{!}{
    \begin{tabular}{l|l|cccc|c}
        \toprule
        Models & Methods & \textbf{STEM} & \textbf{Soc.} & \textbf{Hum.} & \textbf{Other} & \textbf{Avg.} \\
        \midrule
        \multirow{5}{*}{Gemma3-1B} & Full-Adam & 27.63 & 26.84 & 25.12 & 25.14 & 26.04 \\ 
        & LoRA & 26.61 & 25.93 & 24.78 & 25.02 & 25.48\\ 
        & GaLore & 27.40 & 26.71 & 24.80 & 25.69 & 25.99 \\ 
        & APOLLO & 27.47 & 27.14 & 25.61 & 25.76 & \textbf{26.38} \\ 
        &\cellcolor{blue4}\name & \cellcolor{blue4}26.01 & \cellcolor{blue4}26.39 & \cellcolor{blue4}25.44 & \cellcolor{blue4}26.40 & \cellcolor{blue4}25.99 \\ 
        \midrule
        \multirow{5}{*}{LLaMA3.2-3B} & Full-Adam & 46.55 & 65.81 & 49.16 & 63.17 & 55.48 \\
        & LoRA & 47.22 & 63.41 & 49.56 & 61.57 & 54.86 \\ 
        & GaLore & 46.62 & 65.58 & 48.37 & 63.33 & 55.22 \\ 
        & APOLLO & 46.16 & 65.91 & 49.03 & 62.80 & 55.29 \\ 
        &\cellcolor{blue4}\name & \cellcolor{blue4}46.72 & \cellcolor{blue4}65.19 & \cellcolor{blue4}49.29 & \cellcolor{blue4}62.77 & \cellcolor{blue4}\textbf{55.33}\\
        \midrule
        \multirow{5}{*}{Qwen2.5-7B} & Full-Adam & \multicolumn{4}{c|}{OOM} & -\\
        & LoRA & 67.69 & 82.91 & 66.29 & 75.72 & 72.41 \\ 
        & GaLore & 68.32 & 82.87 & 66.08 & 76.10 & 72.55 \\ 
        & APOLLO & 68.49 & 82.97 & 65.80 & 76.65 & 72.65 \\ 
        &\cellcolor{blue4}\name & \cellcolor{blue4}69.12 & \cellcolor{blue4}83.49 & \cellcolor{blue4}66.48 & \cellcolor{blue4}76.28 & \cellcolor{blue4}\textbf{73.04}\\ 
        \bottomrule
    \end{tabular}
    }

\end{table}

\begin{table*}[!th]
    \centering
    \caption{\textbf{Evaluating \name for memory-efficient fine-tuning on the GLUE benchmark (higher is better)}, using a pre-trained RoBERTa-Large model. We report overall (matched and mismatched) accuracy for MNLI, Matthew’s correlation coefficient for CoLA, Pearson correlation for STS-B, and classification accuracy for all other tasks.}
    \label{tab:fine_tuning_glue}
    \resizebox{0.9\linewidth}{!}{
    \begin{tabular}{l|c|cccccccc|c}
    \toprule
       Methods & Memory & \textbf{CoLA} & \textbf{STS-B} & \textbf{MRPC} & \textbf{RTE} & \textbf{SST2} & \textbf{MNLI} & \textbf{QNLI} & \textbf{QQP} & \textbf{Avg.} \\
       \midrule
       Full-Adam & 2.13G & 64.85 & 91.60 & 92.79 & 78.81 & 96.44 & 90.51 & 94.43 & 91.90 & 87.66 \\
       \midrule
       LoRA & 0.73G & \textbf{64.32} & 90.68 & 91.39 & 77.72 & 95.98 & \textbf{90.57} & \textbf{94.78} & \textbf{90.93} & 87.04 \\
       GaLore & 0.72G & 62.52 & 91.18 & 90.94 & 77.11 & \textbf{96.10} & 90.27 & 94.21 & 89.99 & 86.54 \\
       APOLLO & 0.72G & 61.13 & 91.66 & 92.14 & 80.14 & 95.06 & 89.65 & 93.61 & 89.27 & 86.58 \\
       \rowcolor{blue4}
       \name & 0.71G & 64.30 & \textbf{92.44} & 
       \textbf{92.33} & \textbf{83.03} & 95.75 & 89.74 & 93.88 & 89.65 & \textbf{87.64} \\
       \bottomrule
    \end{tabular}
    }

\end{table*}

\textbf{Baselines.}
Consistent with the pre-training experiments, we compare \name against Full-Adam, GaLore, and APOLLO. Additionally, we include \textbf{LoRA}~\cite{hu2021lora} as a baseline. All methods adopt the same hyperparameter strategy to ensure fair comparisons. \name’s level $l$ is set to $8$ for both GLUE and MMLU tasks, except for Gemma3-1B where $l=7$. For the other baselines, the rank is set to 4 on GLUE and 8 on MMLU, ensuring comparable memory usage. Since the level $l$ is already small, we don't run $\namec$.

\textbf{Main Results.}
As reported in Tables~\ref{tab:fine_tune_mmlu} and ~\ref{tab:fine_tuning_glue}, \name achieves comparable performance relative to baselines across a range of downstream LLM tasks. These results highlight \name’s effectiveness beyond pre-training and suggest it can serve as a unified, memory-efficient optimization method applicable throughout the LLM training.

\section{Extra Investigation and Ablation Study}
In this section, we present additional empirical studies to further evaluate the robustness and generalization of \name. Our analysis includes: (i) Pre-training Qwen2.5 \cite{yang2024qwen2.5} models at multiple scales; (ii) Assessing \name under long-sequence training regimes and evaluating its performance with large total training-token budgets; (iii) Ablation study of the hyperparameter $l$; (iv) an ablation study of the residual term in Eq.~\eqref{eq:foam_m_v_res}; (v) Evaluation \name under 8-bit quantization. Beyond these core experiments, the Appendix includes additional studies on GPT~\cite{radford2019gpt2} and DeBERTa~\cite{he2021debertadecodingenhancedbertdisentangled} models, a further ablation of the scaling coefficient $\alpha$, and a comparison of \name with Adam using a similar module-wise learning rate design. These results demonstrate that \name is robust to the hyperparameter $\alpha$ and achieves performance on par with Adam when using the same modul learning rate design.

\paragraph{Pre-training Qwen on C4.} 

\begin{table}[!ht]
    \centering
    \setlength{\tabcolsep}{12pt}
    \renewcommand{\arraystretch}{0.9}
    \caption{Final validation PPL for pre-training Qwen models on C4.}
    \label{tab:validation_qwen_60m_350m}
    \begin{tabular}{l|ccc}
    \toprule
        Methods & {\textbf{60M}} & {\textbf{130M}} & {\textbf{350M}}  \\
        \midrule
        Full-Adam & 29.71 & 22.82 & 16.97 \\
        Adam-mini & 30.04 & 24.12 & 17.55 \\
        Muon & 28.93 & 22.16 & 16.78 \\
        \midrule
        GaLore-1/4 & 33.22 & 25.63 & 19.58 \\
        APOLLO-1/4 & 30.00 & 23.43 & 16.83 \\
        \cellcolor{blue4}\name-2 & \cellcolor{blue4}\textbf{28.57} & \cellcolor{blue4}\textbf{21.25} &  \cellcolor{blue4}\textbf{15.80} \\
        \midrule
        APOLLO-Mini & 32.05 & 23.34 & 16.97 \\
        GWT-Mini & 31.88 & 23.31 & 17.60 \\
        \cellcolor{blue2}\namec & \cellcolor{blue2}\textbf{27.98} & \cellcolor{blue2}\textbf{21.75} & \cellcolor{blue2}\textbf{15.99} \\
         \bottomrule
    \end{tabular}
\end{table}

To assess the generalization capability of \name across different model families, we also pre-train Qwen2.5 \cite{yang2024qwen2.5} models on the C4 dataset, following the same experimental setup in LLaMA experiments. Model sizes range from 60M to 350M parameters, and the final validation PPL results are summarized in Table~\ref{tab:validation_qwen_60m_350m}. Because these model sizes align with those in Table~\ref{tab:validation_llama_60m-1b}, their estimated memory footprints are similar. As shown in the results, both \name and \namec consistently outperform all baseline optimizers, demonstrating that \name generalizes effectively beyond LLaMA architectures.

\begin{figure*}[!th]
    \centering
    \begin{subfigure}[b]{0.33\linewidth}
        \centering
        \includegraphics[width=\linewidth, height=0.69\textwidth]{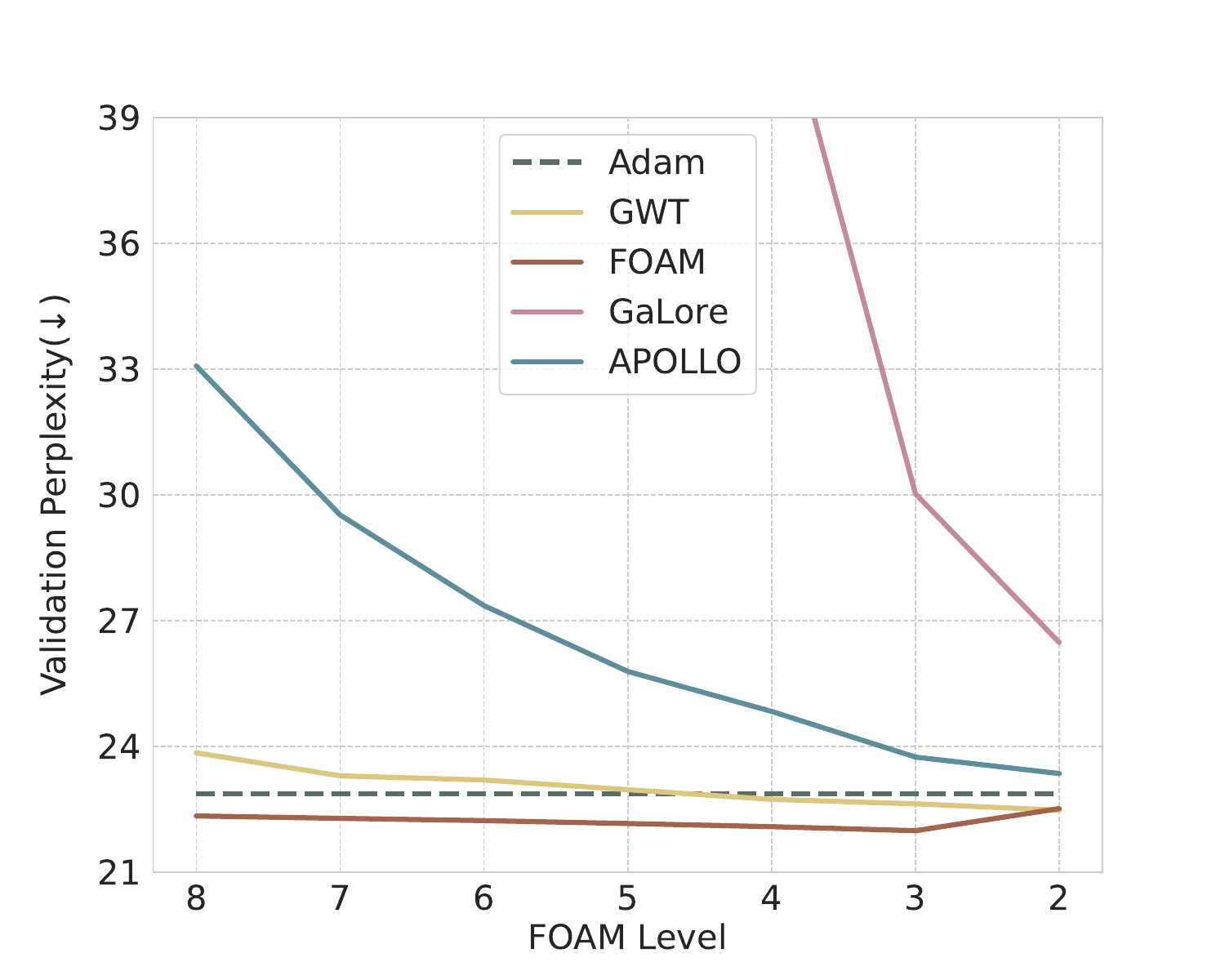}
        \caption{\name with varing level $l$.}
        \label{fig:llama_130m_fold_level}
    \end{subfigure}
    \hfill
    \begin{subfigure}[b]{0.33\linewidth}
        \centering
        \includegraphics[width=\linewidth, height=0.69\textwidth]{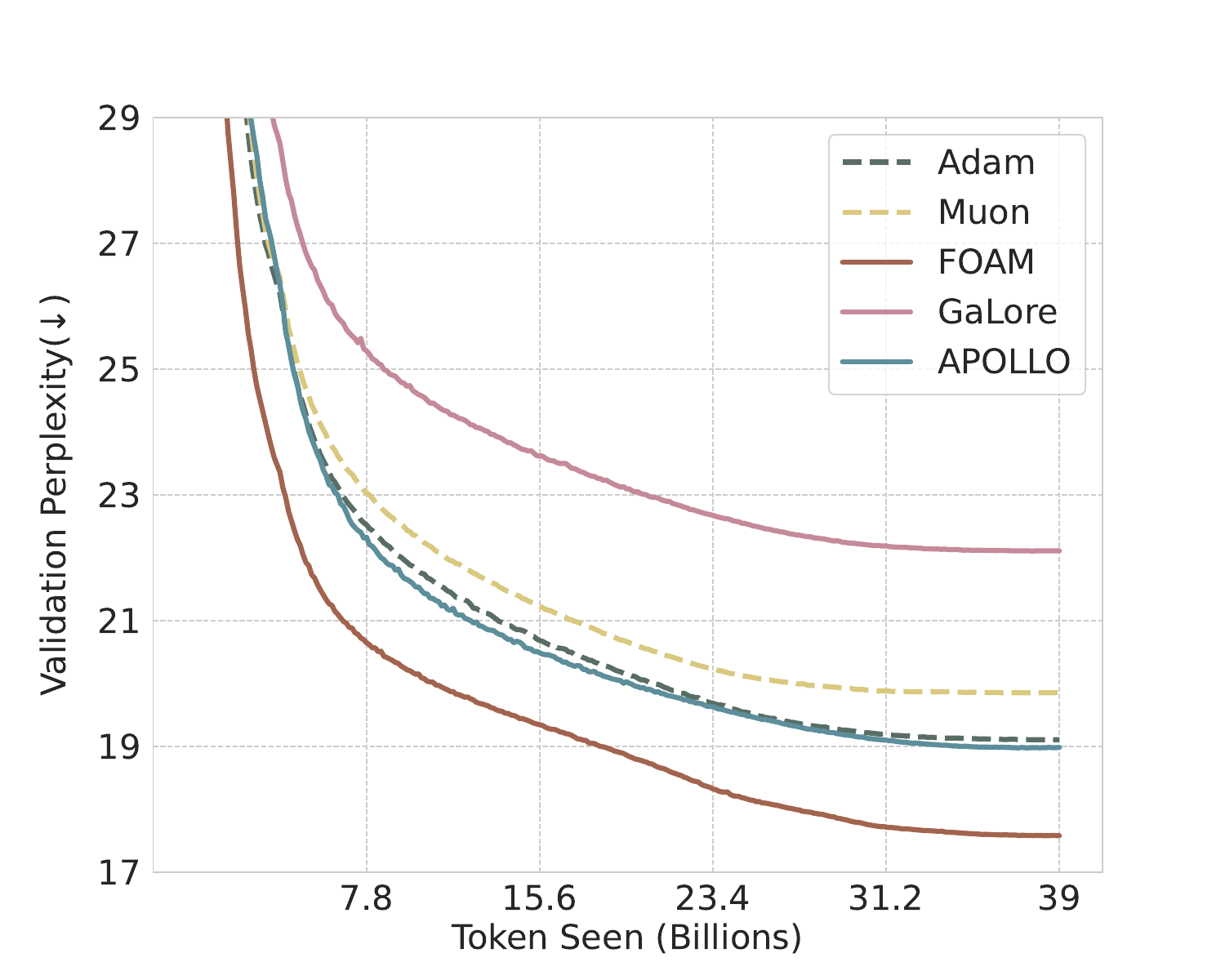}
        \caption{Overtraining LLaMA-130M.}
        \label{fig:llama_130m_over_train}
    \end{subfigure}
    \hfill
    \begin{subfigure}[b]{0.33\linewidth}
        \centering
        \includegraphics[width=\linewidth, height=0.69\textwidth]{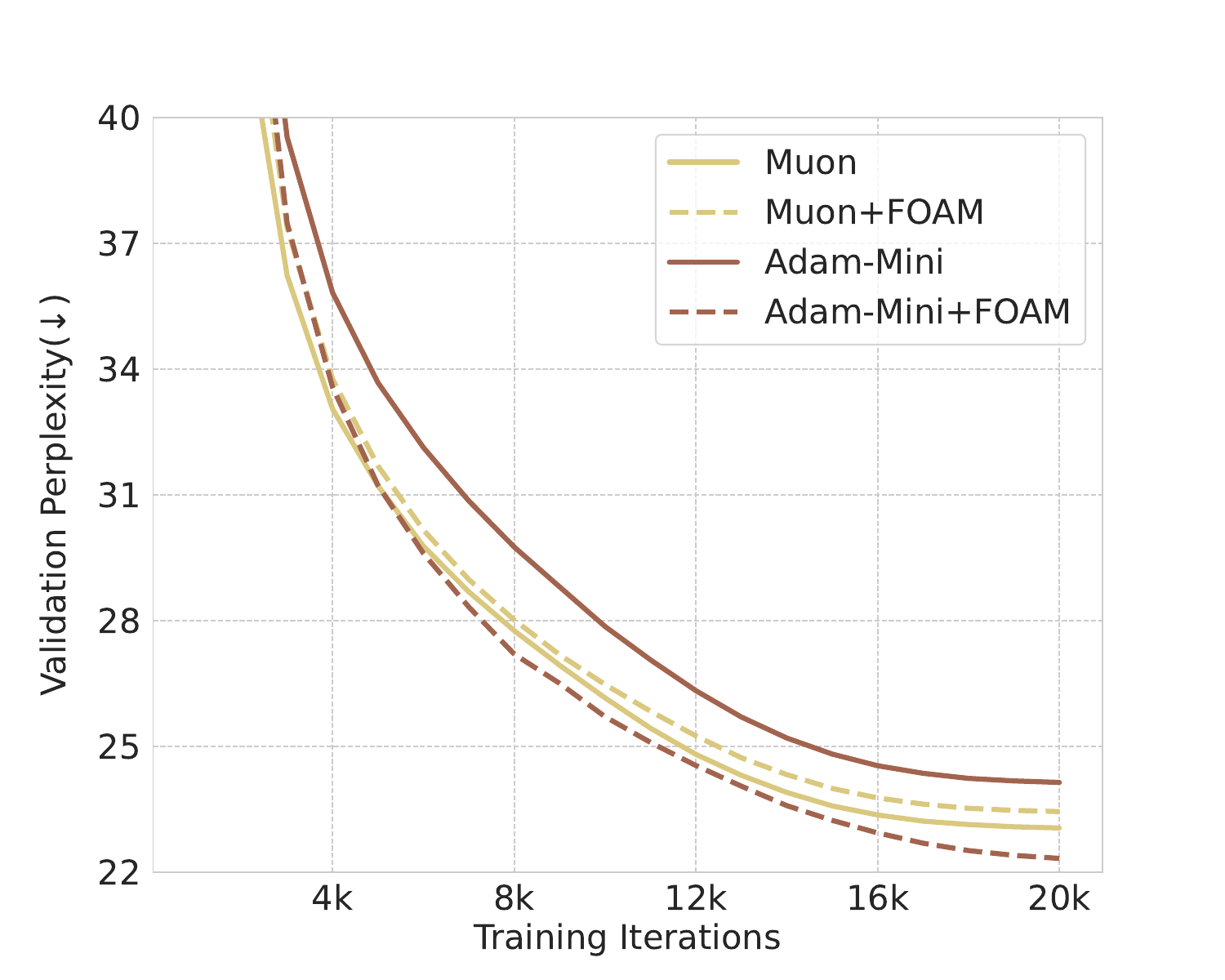}
        \caption{\name with Adam-Mini and Muon.}
        \label{fig:llama_130m_fold_optimizer}
    \end{subfigure}
    \caption{\textbf{Additional Investigation of \name.} \textit{(a)} Impact of the \name level $l$: \name exhibits strong robustness across varying memory constraints. \textit{(b)} Extended training of LLaMA-130M on 39B tokens. \textit{(c)} Integration of \name with Adam-Mini and Muon.}
    \label{fig:main1}
\end{figure*}

\paragraph{Long-context Training.}  

\begin{table}[!ht]
    \centering
    \setlength{\tabcolsep}{6pt}
    \renewcommand{\arraystretch}{0.95}
    \caption{Performance comparison under different sequence lengths.}
    \small
    \label{tab:validation_llama_60m_350m_seq_length}
    \begin{tabular}{c|l|ccc}
    \toprule
        Sequence length & Methods &{\textbf{60M}} & {\textbf{130M}} & {\textbf{350M}}  \\
        \midrule
        \multirow{6}{*}{512} & Full-Adam  &31.54 & 23.77 & 17.98\\
        & Muon  & 30.36 & 23.48 & 17.62 \\
        & Adam-mini & 31.91 & 24.77 & 18.91 \\
        & GaLore-1/4 & 35.25 & 27.19 & 19.92 \\
        & APOLLO-1/4 & 32.02 & 24.04 & 17.26 \\
        & \cellcolor{blue4}\name-2 & \cellcolor{blue4}\textbf{29.26} & \cellcolor{blue4}\textbf{22.47} & \cellcolor{blue4}\textbf{16.34} \\
        \midrule
        \multirow{6}{*}{1024} & Full-Adam & 35.07 & 26.34 & 20.72 \\ 
        & Muon & 33.21 & 26.11 & 19.60 \\
        & Adam-mini & 36.30 & 28.41 & 21.47 \\
        & GaLore-1/4 & 38.09 & 29.51 & 21.73 \\
        & APOLLO-1/4 & 34.04 & 25.93 & 18.77 \\
        & \cellcolor{blue4}\name-2 & \cellcolor{blue4}\textbf{31.69} & \cellcolor{blue4}\textbf{24.38} & \cellcolor{blue4}\textbf{17.75} \\
         \bottomrule
    \end{tabular}
\end{table}

Long-context training is essential for improving the contextual reasoning capabilities of LLMs. To assess \name's generalization and efficiency under extended context windows, we pre-train LLaMA models on the C4 dataset using a range of sequence lengths. In all settings, the total number of tokens per batch is kept fixed to ensure a consistent computational budget and comparable token throughput. As shown in Table~\ref{tab:validation_llama_60m_350m_seq_length}, although all optimizers experience some degradation in PPL as the sequence length increases, \name consistently achieves the best performance. These results demonstrate \name's robustness and stability when training with long input sequences.

\paragraph{Impact of the \name Level $l$.}
We perform an ablation study to assess how the \name level $l$ influences performance. For comparison, we include GaLore, APOLLO, and GWT as baselines, configuring their projection ranks and compression levels to roughly match \name's memory usage under each setting. As illustrated in Figure~\ref{fig:llama_130m_fold_level}, GaLore, APOLLO, and GWT exhibit notable degradation in final validation PPL when operating at lower memory consumptions (corresponding to higher $l$ values). In contrast, \name remains remarkably stable: its performance is largely insensitive to variations in $l$, consistently outperforming Adam even under aggressive memory budgets. These results highlight \name's robustness and its ability to maintain high optimization quality across a wide range of fold levels.

\paragraph{Overtraining LLaMA on C4.}
Because \name, GaLore, and APOLLO generally use larger default learning rates than Full-Adam, they may be more susceptible to instability during prolonged training. To assess \name's resilience in such overtraining regimes, we extend the training of LLaMA-130M to 39B tokens-approximately 300 tokens per parameter, roughly 15× the Chinchilla compute allocation \cite{hoffmann2022scalinglaw}. As shown in Figure~\ref{fig:llama_130m_over_train}, \name remains stable and continues to improve, demonstrating strong robustness even under extreme training durations. 

\paragraph{Integration of \name with Adam-Mini and Muon.}
\begin{figure}[!ht]
    \centering
    \includegraphics[width=0.98\linewidth,height=0.32\textwidth]{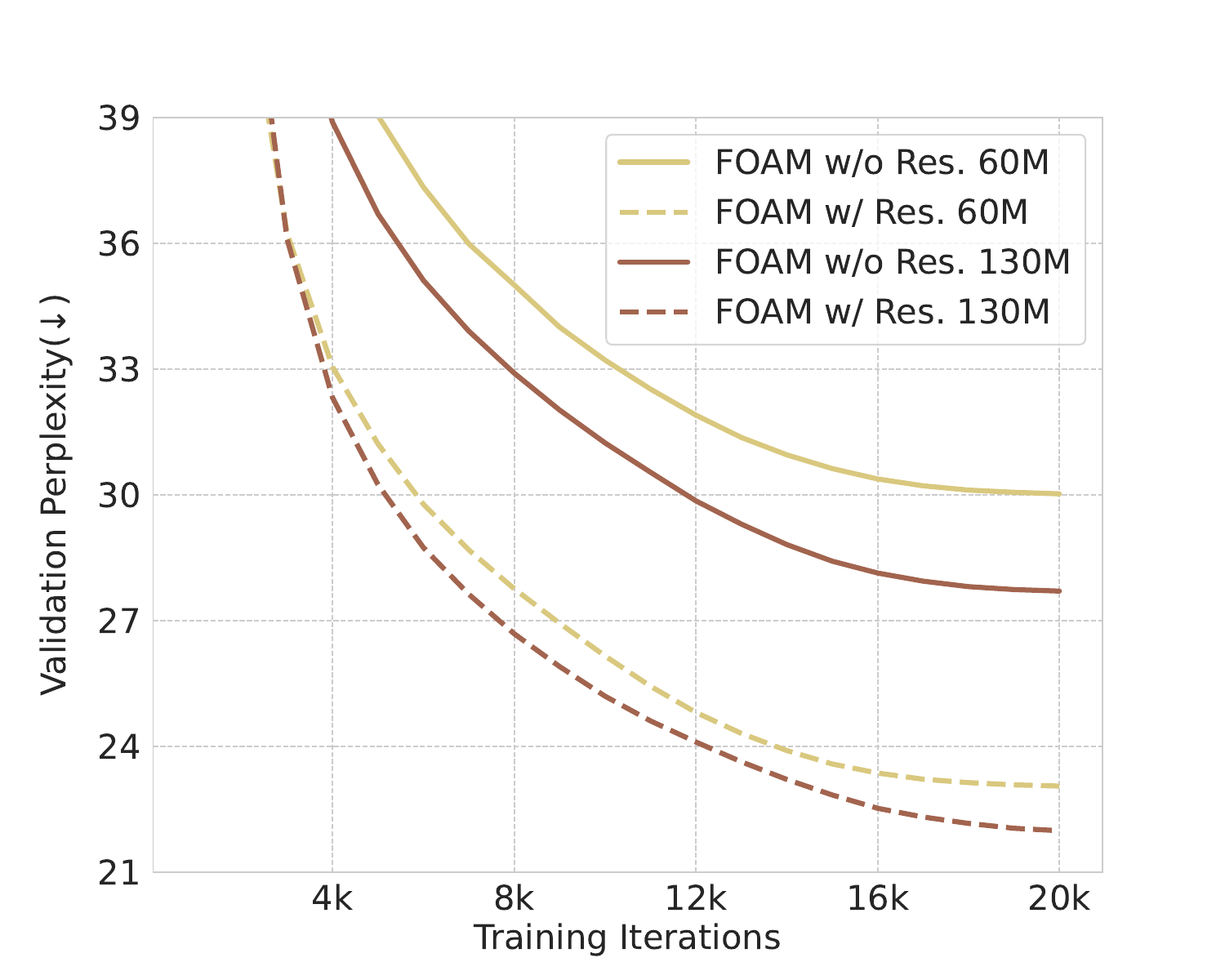}
    \caption{Validation PPL with or without residual.}
    \label{fig:llama_res_60m_130m}
\end{figure}

We further show that \name can be seamlessly combined with other optimizers and memory-saving techniques. To illustrate this flexibility, we integrate \name into Adam-Mini \cite{zhang2024adammini} and Muon \cite{liu2025muon}, and pre-train a LLaMA-130M model. The resulting PPL learning curves, shown in Figure~\ref{fig:llama_130m_fold_optimizer}, indicate that \name consistently matches or exceeds the performance of the original optimizers. These results highlight \name's versatility and effectiveness as a plug-and-play, memory-efficient optimization solution applicable across diverse optimization strategies.

\paragraph{Impact of the Residual.}
We conduct an ablation study on the residual term in Eq.~\eqref{eq:foam_m_v_res} by pre-training the LLaMA-60M and 130M models with and without $R_{t}$. In addition, we measure the cosine similarity between the update magnitudes produced by \name and those of Adam, comparing versions with and without the residual term (Appendix Figure~\ref{fig:cosine_similarity}). We present the PPL learning curves in Figure~\ref{fig:llama_res_60m_130m}. When the residual term is omitted, all parameters within the same block receive identical updates, with any intra-block differences attributable solely to initialization. Our results show that including the residual term significantly accelerates convergence and yields lower final validation PPL, underscoring its critical role in ensuring effective optimization with \name.

\begin{table}[!ht]
    \centering
    \caption{Evaluating \name under 8-bit state quantization.}
    \label{tab:validation_llama_60m-350m_int8}
    \resizebox{0.92\linewidth}{!}{\begin{tabular}{l|cc|cc|cc}
    \toprule
        \multirow{2}{*}{Methods} & \multicolumn{2}{c}{\textbf{60M}} & \multicolumn{2}{c}{\textbf{130M}} & \multicolumn{2}{c}{\textbf{350M}} \\
        & PPL & Mem. & PPL & Mem. & PPL & Mem. \\
        \midrule
        8-bit-Adam & 30.78 & 0.22G & 23.30 & 0.53G & 17.99 & 1.46G\\
        \midrule
        8-bit-GaLore & 34.88 & 0.18G & 25.53 & 0.38G & 19.79 & 0.92G\\
        \cellcolor{blue4}8-bit-\name-2 & \cellcolor{blue4}\textbf{28.86} & \cellcolor{blue4}0.18G &  \cellcolor{blue4}\textbf{22.66} & \cellcolor{blue4}0.36G & \cellcolor{blue4}\textbf{16.02} & \cellcolor{blue4}0.86G \\
        \midrule
        Training Tokens & \multicolumn{2}{|c|}{1.3B} & \multicolumn{2}{|c|}{2.6B} & \multicolumn{2}{|c}{7.8B} \\
         \bottomrule
    \end{tabular}}
\end{table}

\paragraph{\name under Int8 Quantization.}
Low-bit quantization has become a widely adopted strategy for reducing memory usage in modern LLM training. To assess the robustness of \name's compression mechanism under low-precision settings, we incorporate int8-quantized optimizer states \cite{dettmers2022llmint8} and pre-train LLaMA models ranging from 60M to 350M. We compare against 8-bit Adam and 8-bit GaLore. Since the low-precision version of APOLLO relies on additional quantization techniques \cite{zhang2024qgalore}, it is excluded from this evaluation. As shown in Table~\ref{tab:validation_llama_60m-350m_int8}, \name maintains superior performance under int8 quantization, demonstrating the resilience and effectiveness of its compression in low-precision regimes.

\section{Conclusion}
In this paper, we propose \name, a memory-efficient training strategy suitable for both pre-training and fine-tuning. \name preserves the structural information in the gradient matrix through blocked averaging and reconstructs lost information by a residual correction at each step. This enables high-performance LLM training with reduced memory consumption. Our theoretical analysis shows that \name retains the convergence rate of adaptive optimizers. Extensive comparisons with existing methods demonstrate that \name effectively reduces memory overhead, accelerates convergence, and improves training speed. Additionally, we show that \name is compatible with optimizers beyond Adam. Overall, our approach offers an effective solution to the optimizer-state memory bottleneck in LLM training, providing a complementary alternative to low-rank projection techniques for memory-efficient optimization.

\section*{Impact Statement}
This paper presents work whose goal is to advance the field of Machine Learning. There are many potential societal consequences of our work, none of which we feel must be specifically highlighted here.

\section*{Acknowledgements}
Tao Sun is supported in part by the  National Natural Science Foundation of China (Grant Nos. 62522610,  62376278), and NUDT Foundational Research Funding (JS25-02).

\bibliography{reference}
\bibliographystyle{icml2026}
\newpage
\appendix
\onecolumn
\noindent\rule{\textwidth}{3pt}
\begin{center}
	{\LARGE \bf Appendix for \vspace{1.2ex}\\
	\fontsize{11.5pt}{\baselineskip}\selectfont FOAM: Blocked State Folding for Memory-Efficient LLM Training}
\end{center}
\noindent\rule{\textwidth}{1.5pt}

\section{Lemmas and Proofs}
In this section, we present the proofs of the theorems discussed in the main text. Before beginning our proof, we first establish some useful lemmas to facilitate the process.

\subsection{Additional Lemmas}
\begin{lemma}\label{lemma:l2_norm_P}
    Denote $P^{(l)} = A^{(l)}E^{(l)}$, as $A^{(l)}, E^{(l)}$ defined in Eq.~\eqref{eq:define_A} and Eq.~\eqref{eq:define_E}, we have
    \begin{equation}
        \nonumber
        \left\|P^{(l)}\right\|_2 =  \rho(P^{(l)}) = 1,\text{and } \left\|I_{n\times n}-P^{(l)}\right\|_2= \rho(I_{n\times n}-P^{(l)}) = 1,
    \end{equation}
    and
    \begin{equation}
        \nonumber 
        \delta_{l} = \max_{1 \leq t \leq T} \frac{\left\|G_{t}\left(I_{n\times n} - P^{(l)}\right)\right\|}{\|G_{t}\|} =\max_{1 \leq t \leq T} \frac{\|R_{t}\|}{\|G_t\|}\leq 1
    \end{equation}
    where $I_{n\times n} \in \mathbb{R}^{n\times n}$ denotes the Identity matrix, and $\rho(\cdot)$ represents the spectral radius.
\end{lemma}

\begin{proof}
    By the definition of $E^{(l)}$ and $A^{(l)}$, we have
    \begin{equation}
        \nonumber
        E^{(l)}=2^{l} \cdot \left(A^{(l)}\right)^{T}, \quad \ E^{(l)}A^{(l)} = I_{n'\times n'}.
    \end{equation}
    where $n'=\frac{n}{2^{l}}$. By the definition of $P^{(l)}$, we have
    \begin{equation}
        \nonumber
        P^{(l)} = A^{(l)}E^{(l)} = 2^{l}\cdot A^{(l)} \left(A^{(l)}\right)^{T}, \quad \left(P^{(l)}\right)^{T} = \left(A^{(l)}E^{(l)}\right)^{T} = 2^{l} \cdot \left(A^{(l)} \left(A^{(l)}\right)^{T}\right)^{T}= 2^{l}\cdot A^{(l)} \left(A^{(l)}\right)^{T},
    \end{equation}
    this proves that $P^{(l)}$ is a symmetric matrix. We consider
    \begin{equation}
        \nonumber
        \left(P^{(l)}\right)^{2} = \left(A^{(l)}E^{(l)}\right)\left(A^{(l)}E^{(l)}\right) = A^{(l)}\left(E^{(l)}A^{(l)}\right)E^{(l)} = A^{(l)}I_{n'\times n'}E^{(l)} = A^{(l)}E^{(l)} = P^{(l)},
    \end{equation}
    this demonstrates that $P^{(l)}$ is an idempotent matrix. Suppose $\lambda$ is an eigenvalue of $P^{(l)}$ with $\mathbf{v}$ as its associated eigenvector, then we have
    \begin{equation}
        \nonumber
        \left(P^{(l)}\right)^2\mathbf{v} = P^{(l)}\mathbf{v} \longrightarrow \lambda^2\mathbf{v} = \lambda \mathbf{v} \longrightarrow \lambda (\lambda-1) = 0.
    \end{equation}
    Therefore, the eigenvalues of $P^{(l)}$ can only be $0$ or $1$ and $\rho(P^{(l)}) = 1$, this gives us $\|P^{(l)}\|_{2}=1$.

    Next, we start to analyse $I_{n\times n}-P^{(l)}$. In the previous proof, we have established that $P^{(l)}$ is a symmetric matrix. Hence, we can see that  $I_{n\times n}-P^{(l)}$ is also a symmetric matrix. Consider $\left(I_{n\times n}-P^{(l)}\right)^2$, we have
    \begin{equation}
        \nonumber
        \left(I_{n\times n}-P^{(l)}\right)^2 = \left(I_{n\times n}\right)^2 - 2 I_{n\times n}P^{(l)} + \left(P^{(l)} \right)^{2} = I_{n\times n} - P^{(l)}.
    \end{equation}
    This gives us $I_{n\times n} - P^{(l)}$ is an idempotent matrix. Therefore, the eigenvalues of  $I_{n\times n} - P^{(l)}$ can only be $0$ or $1$ and $\rho(I_{n\times n} - P^{(l)})=1$. Thus, we can see that
    \begin{equation}
        \nonumber
        \left\|G_{t}\left(I_{n\times n} - P^{(l)}\right)\right\| \leq \|G_{t}\|\left\|I_{n\times n} - P^{(l)}\right\|_{2} \leq \|G_{t}\| \rightarrow \delta_{l} \leq 1
    \end{equation}
    The proof is completed.
\end{proof}

The following lemma bounds the norm of the difference between the \name update direction $M_{t}$ and the gradient $G_{t}$.

\begin{lemma}\label{lemma:delta_G_norm}
    Let Assumptions \ref{ass:lipschitz}-\ref{ass:bound_gradient} hold, and $\eta_{t}=\eta_0 /\sqrt{t}$. The difference $\Delta_{t} = M_{t} - \nabla f(W_{t})$ can be decomposed as $\Delta_t = \hat{\Delta}_t + \Xi_t$, where $\mathbb{E}[\Xi_t] = 0$ and $\Xi_t$ is the stochastic noise part. For the deterministic part $\hat{\Delta}_t$, there exist constants $b_{1}, b_{2}, b_{3},b_{4},b_{5},b_{6}$ such that
    \begin{equation}
    \nonumber
    \mathbb{E}\left[\|\hat{\Delta}_{t}\|^{2}\right] \leq b_{1} \frac{1}{t} + b_{2} \beta_{1}^{t} + b_{3} \beta_{1}^{2t} + b_{4} \frac{\sigma^2}{t} + b_{5} \sigma^2 \beta_{1}^{t} + b_{6} \sigma^2 \beta_{1}^{2t},
    \end{equation}
\end{lemma}

\begin{proof}
    From Eq.~\eqref{eq:m_v_update_res}, we have
    \begin{equation}
        \nonumber
        \tilde{M}_{t} = \beta_1 \cdot \tilde{M}_{t-1} + (1-\beta_1) \cdot \tilde{G}_{t} = (1-\beta_1) \sum_{j=0}^{t-1} \beta_{1}^{j} \tilde{G}_{t-j}.
    \end{equation}
    Therefore, by the definition of $R_{t}$ and the fact $G_{t} = \nabla f(W_{t}) + \xi_{t}$, 
    \begin{equation}
        \nonumber
        \begin{aligned}
                    \Delta_{t} &= M_{t} - \nabla f(W_{t})  = \tilde{M}_{t} E^{(l)} + R_{t} - G_{t} + \xi_{t}\\&= \tilde{M}_{t} E^{(l)} + \left( G_{t} - G_{t}A^{(l)}E^{(l)}\right) - G_{t} + \xi_{t}= \tilde{M}_{t} E^{(l)} - G_{t}A^{(l)}E^{(l)} + \xi_{t},
        \end{aligned}
    \end{equation}
    and 
    \begin{align}
        \Delta_{t} &= (1-\beta_1) \sum_{j=0}^{t-1} \beta_{1}^{j} \tilde{G}_{t-j}E^{(l)} - \tilde{G_{t}}E^{(l)} + \xi_{t}\nonumber \\&= (1-\beta_1)\sum_{j=1}^{t-1}  \beta_{1}^{j} \tilde{G}_{t-j}E^{(l)} +(1-\beta_{1})\tilde{G}_{t}E^{(l)} - \tilde{G}_{t}E^{(l)} + \xi_{t}\nonumber \\
        &= (1-\beta_1)\sum_{j=1}^{t-1}  \beta_{1}^{j} \tilde{G}_{t-j}E^{(l)} - \beta_{1}\tilde{G}_{t}E^{(l)} + \xi_{t}\nonumber\\&= (1-\beta_1)\sum_{j=1}^{t-1}  \beta_{1}^{j} {G}_{t-j}P^{(l)} - \beta_{1}{G}_{t}P^{(l)}+ \xi_{t}. \label{eq:lemma_delta_1}
    \end{align}
    By applying the geometric series summation formula, we obtain
    \begin{equation}
        \nonumber
        \beta_{1}G_{t}P^{(l)} = (1-\beta_{1}) \sum_{j=1}^{t-1}\beta_{1}^{j}G_{t}P^{(l)} + \beta_{1}^{t}G_{t}P^{(l)}.
    \end{equation}
    Substituting the above expression into Eq.~\eqref{eq:lemma_delta_1}, we have
    \begin{equation}
        \nonumber
        \Delta_{t} = (1-\beta_{1}) \sum_{j=1}^{t-1}  \beta_{1}^{j} \left({G}_{t-j} - G_{t}\right)P^{(l)} - \beta_{1}^{t}G_{t}P^{(l)}+ \xi_{t}.
    \end{equation}
    Given $G_{t} = \nabla f(W_{t}) + \xi_{t}$, this further yields
    \begin{equation}
        \label{eq:lemma_delta_3_term}
        \Delta_{t} = {\underbrace{(1-\beta_{1}) \sum_{j=1}^{t-1}  \beta_{1}^{j}\left( \nabla f(W_{t-j}) - \nabla f(W_{t})\right)P^{(l)}}_{H_{t}} - \beta_{1}^{t}G_{t}P^{(l)}} + \underbrace{(1-\beta_{1}) \sum_{j=1}^{t-1}  \beta_{1}^{j}\left(\xi_{t-j} -\xi_{t} \right)P^{(l)}+ \xi_{t}}_{\Xi_t}.
        \end{equation}
    This gives us
    \begin{equation}
        \nonumber \Delta_{t} = \hat{\Delta}_{t} + \Xi_{t},
    \end{equation}
    where
    \begin{equation}\label{eq:hat_delta}
        \hat{\Delta}_{t}:= {{(1-\beta_{1}) \sum_{j=1}^{t-1}  \beta_{1}^{j}\left( \nabla f(W_{t-j}) - \nabla f(W_{t})\right)P^{(l)}} - \beta_{1}^{t}G_{t}P^{(l)}}
    \end{equation}
    By the Assumption~\ref{ass:unbiased}, we have
    \begin{equation}\nonumber
        \mathbb{E}\left[{(1-\beta_{1}) \sum_{j=1}^{t-1}  \beta_{1}^{j}\left(\xi_{t-j} -\xi_{t} \right)P^{(l)}} + \xi_{t}\right] = \mathbb{E}\left[\Xi_{t}  \right] = 0.
    \end{equation}
    Next, we consider the upperbound of $\|H_{t}\|$, 
    
        \begin{align}
            \|H_{t}\| &= \left\|(1-\beta_{1}) \sum_{j=1}^{t-1}  \beta_{1}^{j}\left( \nabla f(W_{t-j}) - \nabla f(W_{t})\right)P^{(l)}\right\|=(1-\beta_{1}) \sum_{j=1}^{t-1}  \beta_{1}^{j}\left\| \nabla f(W_{t-j}) - \nabla f(W_{t})P^{(l)} \right\| \nonumber \\
            &\leq (1-\beta_{1}) \sum_{j=1}^{t-1}  \beta_{1}^{j}\left\| \nabla f(W_{t-j}) - \nabla f(W_{t})\right\|\|P^{(l)}\|_2 \leq (1-\beta_{1}) \sum_{j=1}^{t-1}  \beta_{1}^{j}L \left\| W_{t-j} - W_{t}\right\|, \label{eq:lemma_delta_t_update_to_t-j}
        \end{align}
   
    where the last inequality we use the $L$-Lipschitz property in Assumption~\ref{ass:lipschitz} and Lemma~\ref{lemma:l2_norm_P}. By the recursion in Eq.~\eqref{eq:adam_updates}, we have
    \begin{equation}
        \label{eq:lemma_delta_update_t_to_j}
        \left\| W_{t-j} - W_{t}\right\| = \left\|\sum_{u=t-j}^{t-1} W_{u+1} - W_{u}\right\|\leq \sum_{u=t-j}^{t-1} \left\|W_{u+1}- W_{u}\right\| = \sum_{u=t-j}^{t-1} \eta_{u}\left\| \frac{M_{u}}{\sqrt{V_{u}}+\epsilon}\right\|
    \end{equation}
    
    By the recursion of $M_{t},V_{t}$ in Eq.~\eqref{eq:m_v_update_res}, and using the fact that $\mathbb{E}[\|G_t\|^2] = \|\nabla f(W_t)\|^2 + \mathbb{E}[\|\xi_t\|^2] \leq C^2 + \sigma^2$ from Assumption \ref{ass:unbiased} and \ref{ass:bound_gradient}, we can bound the second moment of $M_u$ as follows

        \begin{align}
        \mathbb{E}[\|M_{u}\|^2] &= \mathbb{E}\left[\left\|(1-\beta_1) \sum_{j=0}^{u-1} \beta_{1}^{j} {G}_{u-j}P^{(l)} + G_{u}\left(I_{n\times n} - P^{(l)}\right)\right\|^2\right] \nonumber \\&\overset{(a)}{\leq} 2(1-\beta_1)^2 \mathbb{E}\left[\left\|\sum_{j=0}^{u-1} \beta_1^j G_{u-j} P^{(l)}\right\|^2\right] + 2\mathbb{E}\left[|G_u(I_{n\times n}-P^{(l)})|^2\right] \nonumber \\&\overset{(b)}{\leq} 2(1-\beta_1) \sum_{j=0}^{u-1} \beta_1^j \mathbb{E}[\|G_{u-j}\|^2] \|P^{(l)}\|_2^2 + 2\delta_l^2 \mathbb{E}[\|G_u\|^2] \nonumber \\& \overset{(c)}{\leq} 2(C^2 + \sigma^2) + 2\delta_l^2 (C^2 + \sigma^2) = 2(1 + \delta_l^2)(C^2 + \sigma^2), \label{eq:m_norm_sq_bound}
        \end{align}

    where $(a)$ uses $(a+b)^2 \leq 2a^2 + 2b^2$, $(b)$ applies Jensen's inequality to the first term and Lemma~\ref{lemma:l2_norm_P} to the second, and $(c)$ follows from the bounded second moment of $G_t$. We denote $C_M = \sqrt{2(1+\delta_l^2)(C^2+\sigma^2)}$ as the upper bound for $\sqrt{\mathbb{E}[\|M_u\|^2]}$. 
    
    For the second moment term $V_u$, since $V_u$ is a weighted sum of squared stochastic gradients, we have
    \begin{align}
    \mathbb{E}[\|{V_{u}}\|^2] &= \mathbb{E}\left[\left\|(1-\beta_2) \sum_{j=0}^{u-1} \beta_{2}^{j} \left({G}_{u-j}P^{(l)}\right)^{2} + \left(G_{u}\left(I_{n\times n} - P^{(l)}\right)\right)^{2}\right\|\right] \nonumber \\&\leq (1-\beta_2) \sum_{j=0}^{u-1} \beta_{2}^{j} \mathbb{E}[\|G_{u-j}\|^2] \|P^{(l)}\|_2^2 + \mathbb{E}[\|G_u\|^2] \|I_{n\times n}-P^{(l)}\|_2^2 \nonumber 
    \\&\leq (C^2 + \sigma^2) + \delta_l^2 (C^2 + \sigma^2) = (1+\delta_l^2)(C^2 + \sigma^2). \label{eq:v_max_bound}
    \end{align}
    
    It remains straightforward that for the denominator term
        \begin{equation}
        \label{eq:v_frac_norm_revised}
        \frac{1}{\sqrt{V_{t}}+\epsilon} \leq \frac{1}{\epsilon}.
        \end{equation}

    Substituting Eq.~\eqref{eq:m_norm_sq_bound} and Eq.~\eqref{eq:v_frac_norm_revised} into the bound for $\left\| W_{t-j} - W_{t}\right\|$, we obtain the expected distance
    \begin{equation}\nonumber
    \mathbb{E}\left[ \left| W_{t-j} - W_{t}\right| \right] \leq \sum_{u=t-j}^{t-1} \eta_{u} \mathbb{E}\left[\left\|\frac{M_{u}}{\sqrt{V_{u}}+\epsilon}\right\|\right] \leq \frac{C_M}{\epsilon} \sum_{u=t-j}^{t-1} \eta_{u}.
    \end{equation}
        
    Substituting the above result into Eq.~\eqref{eq:lemma_delta_t_update_to_t-j}, and taking the expectation, we have
    \begin{equation}
    \nonumber
    \begin{aligned}
    \mathbb{E}[\|H_{t}\|] &\leq (1-\beta_{1})\frac{(1+\delta_{l})2\sqrt{C^2+\sigma^2} L}{\epsilon} \sum_{j=1}^{t-1} \beta_{1}^{j} \sum_{u=t-j}^{t-1} \eta_{u} \\
    &= (1-\beta_{1})\frac{(1+\delta_{l})2\sqrt{C^2+\sigma^2} L}{\epsilon} \sum_{u=1}^{t-1}\eta_{u}\sum_{j=t-u}^{t-1}\beta_{1}^{j}\\
    &= (1-\beta_{1})\frac{2(1+\delta_{l})\sqrt{C^2+\sigma^2} L}{\epsilon} \sum_{u=1}^{t-1}\eta_{u}\beta_{1}^{t-u}\frac{1-\beta_{1}^{u}}{1-\beta_{1}} \\
    &=\frac{2(1+\delta_{l})\sqrt{C^2+\sigma^2} L}{\epsilon} \sum_{u=1}^{t-1}\eta_{u}\beta_{1}^{t-u}.
    \end{aligned}
    \end{equation}
    Considering the truncation $k=\lfloor \frac{t}{2} \rfloor$, we have
    \begin{equation}
    \label{eq:lemma_delta_truc_2_term_new}
    \mathbb{E}[\|H_{t}\|] \leq \frac{(1+\delta_{l})2\sqrt{C^2+\sigma^2} L}{\epsilon} \sum_{u=1}^{k-1}\eta_{u}\beta_{1}^{t-u} +\frac{(1+\delta_{l})2\sqrt{C^2+\sigma^2} L}{\epsilon} \sum_{u=k}^{t-1}\eta_{u}\beta_{1}^{t-u}.
    \end{equation}
    For the first term,
    \begin{equation}
    \nonumber
    \frac{(1+\delta_{l})2\sqrt{C^2+\sigma^2} L}{\epsilon} \sum_{u=1}^{k-1}\eta_{u}\beta_{1}^{t-u} \leq \frac{(1+\delta_{l})2\sqrt{C^2+\sigma^2} L\eta_{0}\beta_{1}^{\frac{t}{2}}}{(1-\beta_{1})\epsilon}.
    \end{equation}
    For the second term in the inequality of Eq.~\eqref{eq:lemma_delta_truc_2_term_new},
    \begin{equation}
    \nonumber
    \begin{aligned}
    \frac{(1+\delta_{l})2\sqrt{C^2+\sigma^2} L}{\epsilon} \sum_{u=k}^{t-1}\eta_{u}\beta_{1}^{t-u} &\leq \frac{(1+\delta_{l})2\sqrt{C^2+\sigma^2} L}{\epsilon} \eta_{k}\sum_{u=k}^{t-1}\beta_{1}^{t-u} \\
    &\leq \frac{(1+\delta_{l})2\sqrt{C^2+\sigma^2} L}{\epsilon} \frac{1}{\sqrt{\frac{t}{2}}}\frac{\beta_{1}}{1-\beta_{1}} \\
    &= \frac{(1+\delta_{l})2\sqrt{2}\sqrt{C^2+\sigma^2} L\beta_{1}}{(1-\beta_{1})\epsilon\sqrt{t}}.
    \end{aligned}
    \end{equation}
    We then have
    \begin{equation}
    \label{eq:lemma_delta_bound_term_1_new}
    \mathbb{E}[\|H_{t}\|] \leq \frac{(1+\delta_{l})2\sqrt{C^2+\sigma^2} L\eta_{0}}{(1-\beta_{1})\epsilon} \beta_{1}^{t/2}+ \frac{(1+\delta_{l})2\sqrt{2}\sqrt{C^2+\sigma^2} L\beta_{1}}{(1-\beta_{1})\epsilon} \frac{1}{\sqrt{t}}.
    \end{equation}
    For the initialization bias term $\beta_{1}^{t}G_{t}P^{(l)}$ in Eq.~\eqref{eq:lemma_delta_3_term}, we have
    \begin{equation}
    \label{eq:lemma_delta_bound_term_3_new}
    \mathbb{E}\left[ \left\|\beta_{1}^{t}G_{t}P^{(l)}\right\|^2 \right] \leq \beta_{1}^{2t} \mathbb{E}[\|G_{t}\|^2] \|P^{(l)}\|_{2}^2 \leq \beta_{1}^{2t}(C^2 + \sigma^2).
    \end{equation}
    Substituting Eq.~\eqref{eq:lemma_delta_bound_term_1_new} and Eq.~\eqref{eq:lemma_delta_bound_term_3_new} into the definition of $\hat{\Delta}_{t}$, we obtain
    \begin{equation}
    \nonumber
    \begin{aligned}
    \mathbb{E}\left[\|\hat{\Delta}{t}\|^{2}\right] &= \mathbb{E}\left[\left\|H_{t} - \beta_{1}^{t}G_{t}P^{(l)}\right\|^{2}\right] \\
    &\leq 2\mathbb{E}\left[ \|H_{t}\|^{2} \right] + 2\mathbb{E}\left[ \left\|\beta_{1}^{t}G_{t}P^{(l)}\right\|^{2} \right] \\
    &\leq \frac{32(1+\delta_{l})^{2}(C^2+\sigma^2)L^{2}\beta_{1}^{2}}{(1-\beta_{1})^{2}\epsilon^{2}}\frac{1}{t} + \frac{16(1+\delta_{l})^{2}(C^2+\sigma^2)L^{2}\eta_{0}^{2}}{(1-\beta_{1})^{2}\epsilon^{2}}\beta_{1}^{t} + 2(C^2+\sigma^2)\beta_{1}^{2t}.
    \end{aligned}
    \end{equation}
    By rearranging the terms to explicitly separate the impact of the stochastic noise variance $\sigma^2$, we have
    \begin{equation}
    \nonumber
    \mathbb{E}\left[\|\hat{\Delta}_{t}\|^{2}\right] \leq b_{1} \frac{1}{t} + b_{2} \beta_{1}^{t} + b_{3} \beta_{1}^{2t} + b_{4} \frac{\sigma^2}{t} + b_{5} \sigma^2 \beta_{1}^{t} + 2 \sigma^2 \beta_{1}^{2t},
    \end{equation}
    where the constants are defined as follows:
    \begin{equation}
    \nonumber
    \begin{aligned}
    b_{1} &= \frac{32(1+\delta_{l})^{2}C^2 L^{2}\beta_{1}^{2}}{(1-\beta_{1})^{2}\epsilon^{2}}, b_{2} = \frac{16(1+\delta_{l})^{2}C^2 L^{2}\eta_{0}^{2}}{(1-\beta_{1})^{2}\epsilon^{2}}, b_{3} = 2C^2, \\& b_{4} = \frac{32(1+\delta_{l})^{2}L^{2}\beta_{1}^{2}}{(1-\beta_{1})^{2}\epsilon^{2}}, \
      b_{5} = \frac{16(1+\delta_{l})^{2}L^{2}\eta_{0}^{2}}{(1-\beta_{1})^{2}\epsilon^{2}}, b_{6}=2.
    \end{aligned}
    \end{equation}
    The proof is completed.
\end{proof}

\begin{lemma}\label{lemma:one_step_L_expansion}
    Let Assumptions \ref{ass:lipschitz}-\ref{ass:bound_gradient} hold, and $\eta_{t}=\eta_0 /\sqrt{t}$. Then, there exist constants $d_{1},d_{2},d_{3}$ depend on $L,C,\epsilon$, such that
    \begin{equation}
        \nonumber
        d_{1} \eta_{t}\mathbb{E}\left[\left\|\nabla f(W_{t})\right\|\right]^2 \leq \mathbb{E}\left[ f(W_{t})\right] - \mathbb{E}\left[ f(W_{t+1})\right] + d_{2} \eta_{t} \mathbb{E}\left[\|\Delta_{t}\|^{2}\right] + d_{3}\eta_{t}^{2}.
    \end{equation}
\end{lemma}

\begin{proof}
    By the $L-$smoothness, we have
    \begin{equation}
        \label{eq:lemma_lipchitz_step_bound}
        f(W_{t+1}) \leq f(W_{t}) + \langle\nabla f(W_{t}), W_{t+1}-W_{t} \rangle + \frac{L}{2}\|W_{t+1}-W_{t}\|^{2}.
    \end{equation}
    From the recursion in Eq.~\eqref{eq:adam_updates} and the definition of $\Delta_{t}$, we have
    \begin{equation}
        \nonumber
        W_{t+1}-W_{t} = -\eta_{t} \frac{M_{t}}{\sqrt{V_{t}}+\epsilon} = -\eta_{t} \frac{\nabla f(W_{t})+(M_{t}-\nabla f(W_{t}))}{\sqrt{V_{t}}+\epsilon} = -\eta_{t} \frac{\nabla f(W_{t}) + \Delta_{t}}{\sqrt{V_{t}}+\epsilon}.
    \end{equation}
    Substituting into Eq.~\eqref{eq:lemma_lipchitz_step_bound} and taking expectation, we obtain
    \begin{equation}
        \nonumber
        \begin{aligned}
            \mathbb{E}\left[f(W_{t+1})\right] &\leq \mathbb{E}\left[f(W_{t})\right] - \eta_{t}\mathbb{E}\left[\left\langle \nabla f(W_{t}), \frac{\nabla f(W_{t})}{\sqrt{V_{t}}+\epsilon}\right\rangle\right] \\&-\eta_{t} \mathbb{E}\left[\left\langle\nabla f(W_{t}),\frac{\Delta_{t}}{\sqrt{V_{t}}+\epsilon}\right\rangle\right] + \frac{L\eta_{t}^{2}}{2}\mathbb{E}\left[\left\|\frac{M_{t}}{\sqrt{V_{t}}+\epsilon}\right\|^{2}\right]
            \\&=\mathbb{E}\left[f(W_{t})\right] - \eta_{t}\mathbb{E}\left[\left\langle \nabla f(W_{t}), \frac{\nabla f(W_{t})}{\sqrt{V_{t}}+\epsilon}\right\rangle\right] \\&-\eta_{t} \mathbb{E}\left[\left\langle\nabla f(W_{t}),\frac{\hat{\Delta}_{t}}{\sqrt{V_{t}}+\epsilon}\right\rangle\right] - \mathbb{E}\left[\left\langle\nabla f(W_{t}),\frac{\hat{\Xi}_{t}}{\sqrt{V_{t}}+\epsilon}\right\rangle\right]+ \frac{L\eta_{t}^{2}}{2}\mathbb{E}\left[\left\|\frac{M_{t}}{\sqrt{V_{t}}+\epsilon}\right\|^{2}\right].
        \end{aligned}
    \end{equation}
    By Lemma~\ref{lemma:delta_G_norm} we have 
    \begin{equation}
        \nonumber
        \mathbb{E}\left[\left\langle\nabla f(W_{t}),\frac{\hat{\Xi}_{t}}{\sqrt{V_{t}}+\epsilon}\right\rangle\right] = 0.
    \end{equation}
    Thus, we have
    \begin{equation}
        \nonumber
        \begin{aligned}
            \mathbb{E}\left[f(W_{t+1})\right] &\leq \mathbb{E}\left[f(W_{t})\right] - \eta_{t}\mathbb{E}\left[\left\langle \nabla f(W_{t}), \frac{\nabla f(W_{t})}{\sqrt{V_{t}}+\epsilon}\right\rangle\right] \\&-\eta_{t} \mathbb{E}\left[\left\langle\nabla f(W_{t}),\frac{\hat{\Delta}_{t}}{\sqrt{V_{t}}+\epsilon}\right\rangle\right] + \frac{L\eta_{t}^{2}}{2}\mathbb{E}\left[\left\|\frac{M_{t}}{\sqrt{V_{t}}+\epsilon}\right\|^{2}\right] 
        \end{aligned}
    \end{equation}
    By Eq.~\eqref{eq:v_max_bound} and the Assumption \ref{ass:unbiased}, we have 
    \begin{equation}
        \nonumber
        \eta_{t}\mathbb{E}\left[ \left\langle\nabla f(W_{t}), \frac{\nabla f(W_{t})}{\sqrt{V_{t}}+\epsilon}\right\rangle\right]  \geq \eta_{t} \frac{\|\nabla f(W_{t})\|^{2}}{C_{M}+\epsilon}.
    \end{equation}
    Considering the term $\eta_{t} \mathbb{E}\left[\left\langle\nabla f(W_{t}),\frac{\hat{\Delta}_{t}}{\sqrt{V_{t}}+\epsilon}\right\rangle\right]$. From Young's inequality, we have for any $r>0$, 
    \begin{equation}
    \nonumber
        \left\langle\nabla f(W_{t}),\frac{\hat{\Delta}_{t}}{\sqrt{V_{t}}+\epsilon}\right\rangle \leq \frac{r}{2} \left\|\nabla f(W_{t})\right\|^{2} + \frac{1}{2r}\left\|\frac{\hat{\Delta}_{t}}{\sqrt{V_{t}}+\epsilon}\right\|^{2}.
    \end{equation}
    Let $c_{1}=\frac{1}{2C_M+\epsilon},r=\frac{c_{1}}{2}$, we have
    \begin{equation}
        \nonumber
        \left\langle\nabla f(W_{t}),\frac{\hat{\Delta}_{t}}{\sqrt{V_{t}}+\epsilon}\right\rangle \leq \frac{c_{1}}{4}\left\|\nabla f(W_{t})\right\|^{2} + \frac{1}{c_{1}}\left\|\frac{\hat{\Delta}_{t}}{\sqrt{V_{t}}+\epsilon}\right\|^{2}.
    \end{equation}
    Thus,
    \begin{equation}
        \nonumber
        \begin{aligned}
            \mathbb{E}\left[\left\langle \nabla f(W_{t}), \frac{M_{t}}{\sqrt{V_{t}}+\epsilon}\right\rangle\right] &= \mathbb{E}\left[\left\langle \nabla f(W_{t}), \frac{\nabla f(W_{t})}{\sqrt{V_{t}}+\epsilon}\right\rangle\right] + \mathbb{E}\left[\left\langle \nabla f(W_{t}), \frac{\hat{\Delta}_{t}}{\sqrt{V_{t}}+\epsilon}\right\rangle\right] \\
            &{\geq} \mathbb{E}\left[\left\langle \nabla f(W_{t}), \frac{\nabla f(W_{t})}{\sqrt{V_{t}}+\epsilon}\right\rangle\right] - \left|\mathbb{E}\left[\left\langle \nabla f(W_{t}), \frac{\hat{\Delta}_{t}}{\sqrt{V_{t}}+\epsilon}\right\rangle\right]\right| \\ 
            & \geq c_{1}\|\nabla f(W_{t})\|^{2} - \frac{c_{1}}{4}\left\|\nabla f(W_{t})\right\|^{2} - \frac{1}{c_{1}}\mathbb{E}\left[\left\|\frac{\hat{\Delta}_{t}}{\sqrt{V_{t}}+\epsilon}\right\|^{2}\right] \\
            &= \frac{3c_{1}}{4}\left\|\nabla f(W_{t})\right\|^{2} - \frac{1}{c_{1}}\mathbb{E}\left[\left\|\frac{\hat{\Delta}_{t}}{\sqrt{V_{t}}+\epsilon}\right\|^{2}\right],
        \end{aligned}
    \end{equation}
    Therefore, we can see that
    \begin{equation}
        \nonumber
        \begin{aligned}
            \mathbb{E}\left[f(W_{t+1})\right] &\leq \mathbb{E}\left[f(W_{t})\right] - \frac{3c_{1}\eta_{t}}{4}\mathbb{E}\left[\left\|\nabla f(W_{t})\right\|^{2}\right] + \frac{\eta_{t}}{c_{1}}\mathbb{E}\left[\left\|\frac{\hat{\Delta}_{t}}{\sqrt{V_{t}}+\epsilon}\right\|^{2}\right] + \frac{L\eta_{t}^{2}}{2}\mathbb{E}\left[\left\|\frac{M_{t}}{\sqrt{V_{t}}+\epsilon}\right\|^{2}\right] \\
            & \leq \mathbb{E}\left[f(W_{t})\right] - \frac{3c_{1}\eta_{t}}{4}\mathbb{E}\left[\left\|\nabla f(W_{t})\right\|^{2}\right] + \frac{\eta_{t}}{c_{1}\epsilon^{2}}\mathbb{E}\left[\left\|\hat{\Delta}_{t}\right\|^{2}\right] + \frac{L\eta_{t}^{2}}{2\epsilon^{2}}\mathbb{E}\left[\left\|{M_{t}}\right\|^{2}\right] \\
            &\leq \mathbb{E}\left[f(W_{t})\right] - \frac{3c_{1}\eta_{t}}{4}\mathbb{E}\left[\left\|\nabla f(W_{t})\right\|^{2}\right] + \frac{\eta_{t}}{c_{1}\epsilon^{2}}\mathbb{E}\left[\left\|\hat{\Delta}_{t}\right\|^{2}\right] + \frac{L\eta_{t}^{2}}{2}c_{2}^{2},
        \end{aligned}
    \end{equation}
    where $c_{2}$ denotes $\frac{2C_M}{\epsilon}$. By reformulating the inequality above, we obtain
    \begin{equation}
        \nonumber
        \frac{3c_{1}\eta_{t}}{4}\mathbb{E}\left[\left\|\nabla f(W_{t})\right\|^{2}\right] \leq \mathbb{E}\left[f(W_{t})\right] - \mathbb{E}\left[f(W_{t+1})\right] + \frac{\eta_{t}}{c_{1}\epsilon^{2}}\mathbb{E}\left[\left\|\hat{\Delta}_{t}\right\|^{2}\right] + \frac{L\eta_{t}^{2}}{2}c_{2}^{2}.
    \end{equation}
    This gives us
    \begin{equation}
        \nonumber
        d_{1} \eta_{t}\mathbb{E}\left[\left\|\nabla f(W_{t})\right\|^{2}\right] \leq \mathbb{E}\left[ f(W_{t})\right] - \mathbb{E}\left[ f(W_{t+1})\right] + d_{2} \eta_{t} \mathbb{E}\left[\|\hat{\Delta}_{t}\|^{2}\right] + d_{3}\eta_{t}^{2},
    \end{equation}
    where
    \begin{equation}
        \nonumber
        d_{1} = \frac{3}{4(2C_M+\epsilon)},\ d_{2} = \frac{2C_M+\epsilon}{\epsilon^{2}},\ d_{3} = \frac{2LC_M^{2}}{\epsilon^{2}}.
    \end{equation}
    The proof is completed.
\end{proof}

\subsection{Proof of Theorem~\ref{theo:convergence}}
   \begin{proof}From Lemma~\ref{lemma:delta_G_norm}, we have
   \begin{equation}\nonumber
       \mathbb{E}\left[\|\hat{\Delta}_{t}\|^{2}\right] \leq b_{1} \frac{1}{t} + b_{2} \beta_{1}^{t} + b_{3} \beta_{1}^{2t} + b_{4} \frac{\sigma^2}{t} + b_{5} \sigma^2 \beta_{1}^{t} + b_{6} \sigma^2 \beta_{1}^{2t}.
   \end{equation}
   By summing the above formula from $t=1$ to $T$ with weight $\eta_t$, we obtain
   \begin{equation}\nonumber
   \begin{aligned}
       \sum_{t=1}^{T}\eta_{t} \mathbb{E}\left[\|\hat{\Delta}_{t}\|^{2}\right] &\leq \ b_{1}\sum_{t=1}^{T} \frac{\eta_{t}}{t} + b_{2}\sum_{t=1}^{T}\eta_{t}\beta_{1}^{t} + b_{3}\sum_{t=1}^{T}\eta_{t}\beta_{1}^{2t} \\ &+ b_{4}\sigma^2 \sum_{t=1}^{T} \frac{\eta_{t}}{t} + b_{5}\sigma^2 \sum_{t=1}^{T}\eta_{t}\beta_{1}^{t} + b_{6}\sigma^2 \sum_{t=1}^{T}\eta_{t}\beta_{1}^{2t}.
   \end{aligned}
   \end{equation}
   With the fact that $\eta_{t}=\frac{\eta_{0}}{\sqrt{t}}$, it is straightforward that
   \begin{equation}\nonumber
   \begin{aligned}
       &\sum_{t=1}^{T}\frac{\eta_{t}}{t} = \eta_{0}\sum_{t=1}^{T}t^{-3/2} = c_{3} < +\infty, \\&\sum_{t=1}^{T}{\eta_{t}}\beta_{1}^{t} = \eta_{0} \sum_{t=1}^{T} \frac{\beta_{1}^{t}}{\sqrt{t}} = c_{4} < +\infty, \\&\sum_{t=1}^{T}{\eta_{t}}\beta_{1}^{2t} = \eta_0 \sum_{t=1}^{T} \frac{\beta_{1}^{2t}}{\sqrt{t}} = c_{5} < +\infty, \\& \sum_{t=1}^{T}{\eta_{t}}^{2} = \eta_{0}^{2}\sum_{t=1}^{T} \frac{1}{t} \leq \eta_{0}^{2}(1+\log T).
   \end{aligned}
   \end{equation}
    Therefore, we have
        \begin{equation}\nonumber
        \sum_{t=1}^{T}\eta_{t} \mathbb{E}\left[\|\hat{\Delta}{t}\|^{2}\right] \leq (b_{1} + b_4 \sigma^2) c_{3} + (b_{2} + b_5 \sigma^2) c_{4} + (b_{3} + b_6 \sigma^2) c_{5}.
        \end{equation}
    From Lemma~\ref{lemma:one_step_L_expansion}, we have
    \begin{equation}
    \nonumber
    \begin{aligned}
    d_{1} \sum_{t=1}^{T}\eta_{t}\mathbb{E}\left[\left\|\nabla f(W_{t})\right\|^{2}\right] &\leq \mathbb{E}\left[ f(W_{1})\right] - \mathbb{E}\left[ f(W_{T+1})\right] + d_{2} \sum_{t=1}^{T}\eta_{t} \mathbb{E}\left[\|\hat{\Delta}{t}\|^{2}\right] + d_{3}\sum_{t=1}^{T}\eta_{t}^{2} \
    \\& \leq \mathbb{E}\left[ f(W_{1})\right] - f^{*} + d_{2}\left[ (b_{1} + b_{4}\sigma^{2})c_{3} + (b_{2} + b_{5}\sigma^{2})c_{4} + (b_{3} + b_{6}\sigma^{2})c_{5} \right] \\
    & + d_{3} \eta_{0}^{2}(1+\log T).
    \end{aligned}
    \end{equation}
    Dividing both sides of the inequality by $d_{1}\sum_{t=1}^{T}\eta_{t}$, we have
    \begin{equation}\nonumber
    \begin{aligned}
        &\frac{d_{1}\sum_{t=1}^{T}\eta_{t}\mathbb{E}\left[\|\nabla f(W_{t})\|^{2}\right]}{d_{1}\sum_{t=1}^{T}\eta_{t}} \\& \leq \frac{\mathbb{E}\left[ f(W_{1})\right] - f^{*} + d_{2}\left[ (b_{1} + b_{4}\sigma^{2})c_{3} + (b_{2} + b_{5}\sigma^{2})c_{4} + (b_{3} + b_{6}\sigma^{2})c_{5} \right] + d_{3} \eta_{0}^{2}(1+\log T)}{d_{1}\sum_{t=1}^{T}\eta_{t}}.
    \end{aligned}
    \end{equation}
    Using the fact that $\sum_{t=1}^{T}\eta_{t} \geq 2\eta_{0}(\sqrt{T+1}-1)$, and taking the minimum over $t$, we obtain:
    \begin{equation}
    \nonumber
    \begin{aligned}
    \min_{1\leq t\leq T}\mathbb{E}\left[\|\nabla f(W_{t})\|^{2}\right] \leq \ & \frac{\mathbb{E}\left[ f(W_{1})\right] - f^{} + d_2(b_1 c_3 + b_2 c_4 + b_3 c_5) + d_3 \eta_0^2(1 + \log T)}{2 d_1 \eta_0 (\sqrt{T+1}-1)} \\&+ \frac{d_2 \sigma^2 (b_4 c_3 + b_5 c_4 + b_6 c_5)}{2 d_1 \eta_0 (\sqrt{T+1}-1)}.\end{aligned}\end{equation}By substituting the definitions of $d_3$ and $b_i$, and noting that $d_3$ and $b_{1,\dots,6}$ contain the scaling factor $(1+\delta_l^2)$, we observe that all terms in the numerator are either constants or logarithmic in $T$. Thus, we obtain
    \begin{equation}\nonumber
    \min_{1\leq t\leq T}\mathbb{E}\left[\|\nabla f(W_{t})\|^{2}\right] = \mathcal{O}\left( \frac{\log T + \delta_l^2}{\sqrt{T}} \right) + \mathcal{O}\left( \frac{\sigma^2 \log T}{\sqrt{T}} \right).
    \end{equation}
    
    The proof is completed.
\end{proof}

\section{Experimental Details}\label{sec:experimental_details}

\subsection{Hyperparameters}

In this section, we detail the hyperparameters used to reproduce our experimental results. We adopt the hyperparameters of $(\beta_{1}=0.9, \beta_{2}=0.95,\epsilon=1\mathrm{e}{-8})$ across all the tasks for Adam, as these hyperparameter setting is widely used in LLM training \cite{Touvron2023LLaMAOA}. For pre-training LLaMA models, we fine-tune both Adam~\cite{Kingma2014AdamAM} and Muon~\cite{liu2025muon}, selecting the learning rate ($lr$) that achieves the lowest PPL from the set $\{1.0\mathrm{e}{-4},\ 5.0\mathrm{e}{-4},\ 1.0\mathrm{e}{-3},\ 2.5\mathrm{e}{-3},\ 5.0\mathrm{e}{-3},\ 1.0\mathrm{e}{-2}\}$. For GaLore~\cite{zhao2024galore} and APOLLO~\cite{zhu2024apollosgdlikememoryadamwlevel_apollo}, we tune the learning rate within $\{1.0\mathrm{e}{-3},\ 2.5\mathrm{e}{-3},\ 5.0\mathrm{e}{-3},\ 1.0\mathrm{e}{-2}\}$. Given that our experimental configuration is similar to that of prior studies, we fine-tune their scaling factors, within $\{0.25, 0.5, 0.75, 1.0\}$ for GaLore, APOLLO, \name, \namec, and adopt the recommended $\alpha=128$ for APOLLO-Mini. For our method, both \name and \namec use a scale factor of $\alpha = 0.25$. All experiments use BF16 precision to reduce memory consumption and are parallelized using Distributed Data Parallel (DDP) across multiple GPUs with gradient synchronization using PyTorch's \cite{paszke2017pytorch} \textit{torch.distributed} framework.

Following the experimental setups of previous works~\cite{zhao2024galore, zhu2024apollosgdlikememoryadamwlevel_apollo}, we use a batch size of 512 and a sequence length of 256 by default. \name is applied to both the MLP and attention modules~\cite{Vaswani2017AttentionIA}. The learning rate is linearly warmed up over the first 10\% of iterations, followed by a cosine decay scheduler for the remainder of training.

\begin{table}[!th]
    \centering
    \setlength{\tabcolsep}{8pt}
    \renewcommand{\arraystretch}{1.2}
    \caption{Hyperparameters ($lr,\alpha$) for pre-training LLaMA models.}
    \resizebox{\linewidth}{!}{\begin{tabular}{l|cc|cc|cc|cc|cc|cc}
    \toprule
    Models & \multicolumn{2}{c|}{\textbf{60M}} & \multicolumn{2}{c|}{\textbf{130M}} & \multicolumn{2}{c|}{\textbf{350M}} & \multicolumn{2}{c|}{\textbf{1B}} & \multicolumn{2}{c|}{\textbf{3B}} & \multicolumn{2}{c}{\textbf{7B}}\\
    \midrule
    Hyperparameters & $lr$ & $\alpha$ & $lr$ & $\alpha$ & $lr$ & $\alpha$ & $lr$ & $\alpha$ & $lr$ & $\alpha$ & $lr$ & $\alpha$\\
    \midrule
    Full-Adam (8-bit) & 5.0e-3 & - & 1.0e-3 & - &1.0e-3 & -& 5.0e-4 &- & 5.0e-4 & - & 5.0e-4 & -\\
    Muon & 5.0e-3 & - & 2.5e-3 & - &1.0e-3 & -& 1.0e-3 &- & 1.0e-3 & - & 5.0e-4 \\
    Adam-mini & 5.0e-3 & - & 1.0e-3 & - & 5.0e-4 & - & 2.5e-4 & - & -& - &- & -\\
    LDAdamW & 5.0e-3 & - & 1.0e-3 & - & 1.0e-3 & - & 5.0e-4 & - \\
    GaLore & 1.0e-2 & 0.25 & 1.0e-2 & 0.25 & 1.0e-2 & 0.25 & 1.0e-2 & 0.25 & 5.0e-3 & 0.25 & 1.0e-2 & 0.25\\
    APOLLO & 1.0e-2 & 1.0 & 1.0e-2 & 1.0 & 1.0e-2 & 1.0 & 1.0e-2 & 1.0 & 5.0e-3 & 1.0 & 1.0e-2 & 1.0\\
    GWT & 1.0e-2 & 0.25 & 1.0e-2 & 0.25 & 1.0e-2 & 0.25 & 1.0e-2 & 0.25 & 5.0e-3 & 0.25 & - & - \\
    \name & 1.0e-2 & 0.25 & 1.0e-2 & 0.25 & 1.0e-2 & 0.25 & 1.0e-2 & 0.25 & 5.0e-3 & 0.25 & 5.0e-3 & 0.25\\
    APOLLO-Mini & 1.0e-2 & 128 & 1.0e-2 & 128 & 1.0e-2 & 128 & 1.0e-2 & 128 &- & -& -& -\\
    \namec & 1.0e-2 & 0.25 & 1.0e-2 & 0.25 & 1.0e-2 & 0.25 & 1.0e-2 & 0.25 & - & - & - & -\\
    \bottomrule
    \end{tabular}}
    \label{tab:my_label}
\end{table}

Additionally, we present a concise summary of the architectural hyperparameters for the LLaMA (Large Language Model Meta AI)~\cite{Touvron2023LLaMAOA}, Qwen~\cite{yang2024qwen2.5}, and RoBERTa-Large (Robustly Optimized BERT Approach)~\cite{liu2019roberta} models used in the main text. These details are provided in Table~\ref{tab:llama_parameter}.

\begin{table*}[!ht]
    \centering
    \setlength{\tabcolsep}{5pt}
    \renewcommand{\arraystretch}{1.1}
    \caption{Architecture hyperparameters of LLaMA for pre-training. Batch size and training data amount are specified in tokens.}
    \label{tab:llama_parameter}
    \begin{tabular}{l|ccccccc}
    \toprule
    Model & Params & Hidden & Intermediate & Heads & Layers & Iteration & Training tokens \\
    \midrule
    \multirow{5}{*}{LLaMA} & 60M & 512 & 1376 & 8 & 8 & 10K & 1.3B \\
    & 130M & 768 & 2048 & 12 & 12 & 20K & 2.6B \\
    & 350M & 1024 & 2736 & 16 & 24 & 60K & 7.8B \\
    & 1B & 2048 & 5461 & 24 & 32 & 100K & 13.1B \\
    & 3B & 2560 & 6848 & 32& 32 & 120K &15.7B \\
    & 7B & 4096 & 11008 & 32 & 32 & 150K & 19.7B \\
    \midrule
    \multirow{3}{*}{Qwen} & 60M & 576 & 1536 & 8 & 12 & 10K & 1.3B \\
    & 130M & 768 & 2816 & 12 & 14 & 20K & 2.6B \\
    & 350M & 1024& 3328 & 16 & 26 & 60K & 7.8B \\
    \midrule
    {RoBERTa}& 355M& 1024 & 4096 & 16 & 24 & - & - \\
    \bottomrule
    \end{tabular}
\end{table*}

For fine-tuning the RoBERTa-large model~\cite{liu2019roberta} on the GLUE benchmark~\cite{Wang2018GLUEAM}, we conduct a hyperparameter sweep over the learning rate for each method in the range $\{1\mathrm{e}{-5},\ 2.5\mathrm{e}{-5},\ 5\mathrm{e}{-5},\ 7.5\mathrm{e}{-5},\ 1\mathrm{e}{-4},\ 1.5\mathrm{e}{-4},\ 2\mathrm{e}{-4},\ 4\mathrm{e}{-4}\}$, and report the best performance for each task.

For the MMLU fine-tuning task, we follow the learning rate search strategy proposed in~\citet{zhu2024apollosgdlikememoryadamwlevel_apollo}. Specifically, we evaluate each method by sweeping the learning rate over the range $\{1\mathrm{e}{-5},\ 2.5\mathrm{e}{-5},\ 5\mathrm{e}{-5},\ 1\mathrm{e}{-4},\ 1.5\mathrm{e}{-4},\ 2\mathrm{e}{-4}\}$, and report the highest average score achieved. A detailed summary of the hyperparameter settings used for fine-tuning \name on GLUE is provided in Table~\ref{tab:fine-tuning_GLUE_hyperparameters}.

\begin{table*}[!th]
    \centering
    \setlength{\tabcolsep}{6pt}
    \renewcommand{\arraystretch}{1.2}
    \caption{Hyperparameters of fine-tuning RoBERTa-base model on GLUE for \name.}
    \label{tab:fine-tuning_GLUE_hyperparameters}
    \begin{tabular}{l|cccccccc}
    \toprule
         \textbf{Hyperparameters} & \textbf{CoLA} & \textbf{STS-B} & \textbf{MRPC} & \textbf{RTE} & \textbf{SST2} & \textbf{MNLI} & \textbf{QNLI} & \textbf{QQP} \\
        \midrule
         Batch Size & 32 & 16 & 16 & 16 & 16 & 16 & 16 & 16 \\
         Epochs & \multicolumn{8}{c}{3} \\
         $lr$ Scheduler & \multicolumn{8}{c}{Cosine} \\
         Warmup Steps & \multicolumn{8}{c}{10\%} \\
         Where & \multicolumn{8}{c}{All} \\
         Level $l$ & \multicolumn{8}{c}{8} \\
         Scale $\alpha$ & \multicolumn{8}{c}{0.25} \\
         Max Seq. Len. & \multicolumn{8}{c}{256} \\
         $lr$ & 2.0e-4 & 1.5e-4 & 1.5e-4 & 1.0e-4 & 5.0e-5 & 2.5e-5 & 2.5e-5 & 2.5e-5  \\
         \bottomrule
    \end{tabular}
\end{table*}

\begin{table*}[!th]
    \centering
    \setlength{\tabcolsep}{10pt}
    \renewcommand{\arraystretch}{1.2}
    \caption{Hyperparameters of fine-tuning different models on the MMLU benchmark for \name.}
    \label{tab:mmul_used_lr}
    \begin{tabular}{l|ccc}
    \toprule
    \textbf{Hyperparameters} & Gemma3-1B & LLaMA3.2-3B & Qwen2.5-7B   \\
    \midrule
    Scale $\alpha$ & \multicolumn{3}{c}{0.25} \\
    $lr$ Scheduler & \multicolumn{3}{c}{Cosine}\\
    Warmup Steps & \multicolumn{3}{c}{10\%} \\
    Epochs & \multicolumn{3}{c}{3}\\
    Batch Size & \multicolumn{3}{c}{16}\\
    Where  & \multicolumn{3}{c}{All}\\
    N-shots & \multicolumn{3}{c}{5} \\
    Cut-off Len. & \multicolumn{3}{c}{2048} \\
    Level $l$ & 7 & 8 & 8 \\
    $lr$ & 2.0e-4 & 2.5e-5 &  1.0e-5 \\
   
    \midrule
    \end{tabular}
\end{table*}

\subsection{Memory Estimation} \label{sec:memory_throughput}
For memory estimation, we follow the general approach proposed in GaLore \cite{zhao2024galore}. Specifically, we isolate memory overhead attributable to model parameters and optimizer states, while excluding other factors such as batch size, PyTorch's memory caching and fragmentation behavior \cite{paszke2017pytorch}, and runtime training configurations.

As a representative example, the LLaMA-60M model contains approximately 58 million parameters. Using BF16 precision (2 bytes per parameter), this yields a model memory footprint of approximately 0.11 GB. The Adam optimizer maintains two auxiliary states—the first-order moment $M$ and second-order moment $V$—each the same size as the model parameters, resulting in an additional 0.23 GB. The total memory usage with Adam is thus approximately 0.34 GB. For our \name method, the computation formula depends on the number of parameters that use \name and those that do not. Specifically, for LLaMA-60M, there are 25.3M parameters that use \name to save memory, and the remaining 32.77M parameters are updated using Adam. In the case of using \name with level $l=2$, the optimizer states that the parameters are calculated as follows:
\begin{itemize}
    \item The optimizer state for the \name parameters:
    $$25.3\ \text{M}\times2\ \text{Bytes}\times2/4=25.3\ \text{MBytes}.$$
    \item  The optimizer state for the Adam parameters:
    $$32.77\ \text{M}\times2\ \text{Bytes}\times2=131.08\ \text{MBytes}.$$
    \item The memory occupied by the model parameters:
    $$116.14\ \text{MBytes}.$$
\end{itemize}
Therefore, the estimated training memory overhead with \name-2 is
$$
25.3\ \text{MBytes}+131.08\ \text{MBytes}+116.14\ \text{MBytes}=272.52\ \text{MBytes}\approx0.27\ \text{GBytes}
$$
For GaLore and APOLLO, we assume the model weight matrix has shape $m \times n$, where $m < n$. In this case, the optimizer state that includes the projection matrices uses
$$
2\times n\times r+m \times r.
$$
If $m > n$, the optimizer state size becomes:
$$
2\times m\times r+n \times r.
$$
The total estimated memory usage can be obtained by evaluating the model training code during runtime. Due to the additional projection matrices maintained by the optimizer, the actual memory overhead of GaLore and APOLLO remains higher than that of \name-2, even when using a rank as low as $1/4$ of the model size. We present the estimated memory consumption in Table~\ref{tab:momory_estimate}.

\begin{table*}[!ht]
    \centering
    \setlength{\tabcolsep}{12pt}
    \renewcommand{\arraystretch}{1.15}
    \caption{Estimated model/optimizer states memory comsumption for pre-training 60M-1.3B models.}
    \label{tab:momory_estimate}
    \begin{tabular}{l|ccccccccc}
    \toprule
        Methods & 60M & 130M &350M & 1B  \\
        \midrule
        Full-Adam &  0.11G/0.23G & 0.25G/0.51G & 0.68G/1.37G & 2.60G/5.20G  \\
        Muon & 0.11G/0.19G & 0.25G/0.38G & 0.68G/0.92G & 2.60G/3.61G\\
        \midrule
        GaLore-1/4 & 0.11G/0.17G & 0.25G/0.32G & 0.68G/0.70G & 2.60G/2.16G \\
        APOLLO-1/4 & 0.11G/0.17G & 0.25G/0.32G & 0.68G/0.70G & 2.60G/2.16G \\
        \cellcolor{blue4}\name-2 & \cellcolor{blue4}0.11G/0.16G  & \cellcolor{blue4}0.25G/0.29G & \cellcolor{blue4}0.68G/0.62G & \cellcolor{blue4}2.60G/1.85G \\
        \midrule
        GaLore-1/8 & 0.11G/0.15G & 0.25G/0.27G & 0.68G/0.55G & 2.60G/1.55G \\
        APOLLO-1/8 & 0.11G/0.15G & 0.25G/0.27G & 0.68G/0.55G & 2.60G/1.55G \\
        \cellcolor{blue4}\name-3 & \cellcolor{blue4}0.11G/0.14G & \cellcolor{blue4}0.25G/0.25G & \cellcolor{blue4}0.68G/0.46G & \cellcolor{blue4}2.60G/1.37G \\
        \midrule
        APOLLO-Mini & 0.11G/0.13G & 0.25G/0.21G & 0.68G/0.32G & 2.60G/0.60G \\
        GWT-Mini & 0.11G/0.13G & 0.25G/0.21G & 0.68G/0.32G & 2.60G/0.60G \\
        \cellcolor{blue2}\namec & \cellcolor{blue2}0.11G/0.13G & \cellcolor{blue2}0.25G/0.21G & \cellcolor{blue2}0.68G/0.32G & \cellcolor{blue2}2.60G/0.60G \\
         \bottomrule
    \end{tabular}
\end{table*}

\subsection{Experiment Enviroments}
All pre-training experiments were conducted using 4 to 32 NVIDIA RTX 3090 GPUs and 4 NVIDIA H100 GPUs with PyTorch version 2.3.0. Fine-tuning experiments were performed on a single NVIDIA A100 40GB GPU within the LLaMA-Factory framework \cite{zheng2024llamafactory}, using PyTorch version 2.6.0. All experiments use a random seed of 42 for data shuffling.

\section{Addtional Experiment Results}
\begin{figure*}[!ht]
    \centering
    \begin{subfigure}[b]{0.49\linewidth}
        \centering
        \includegraphics[width=\linewidth, height=0.7\textwidth]{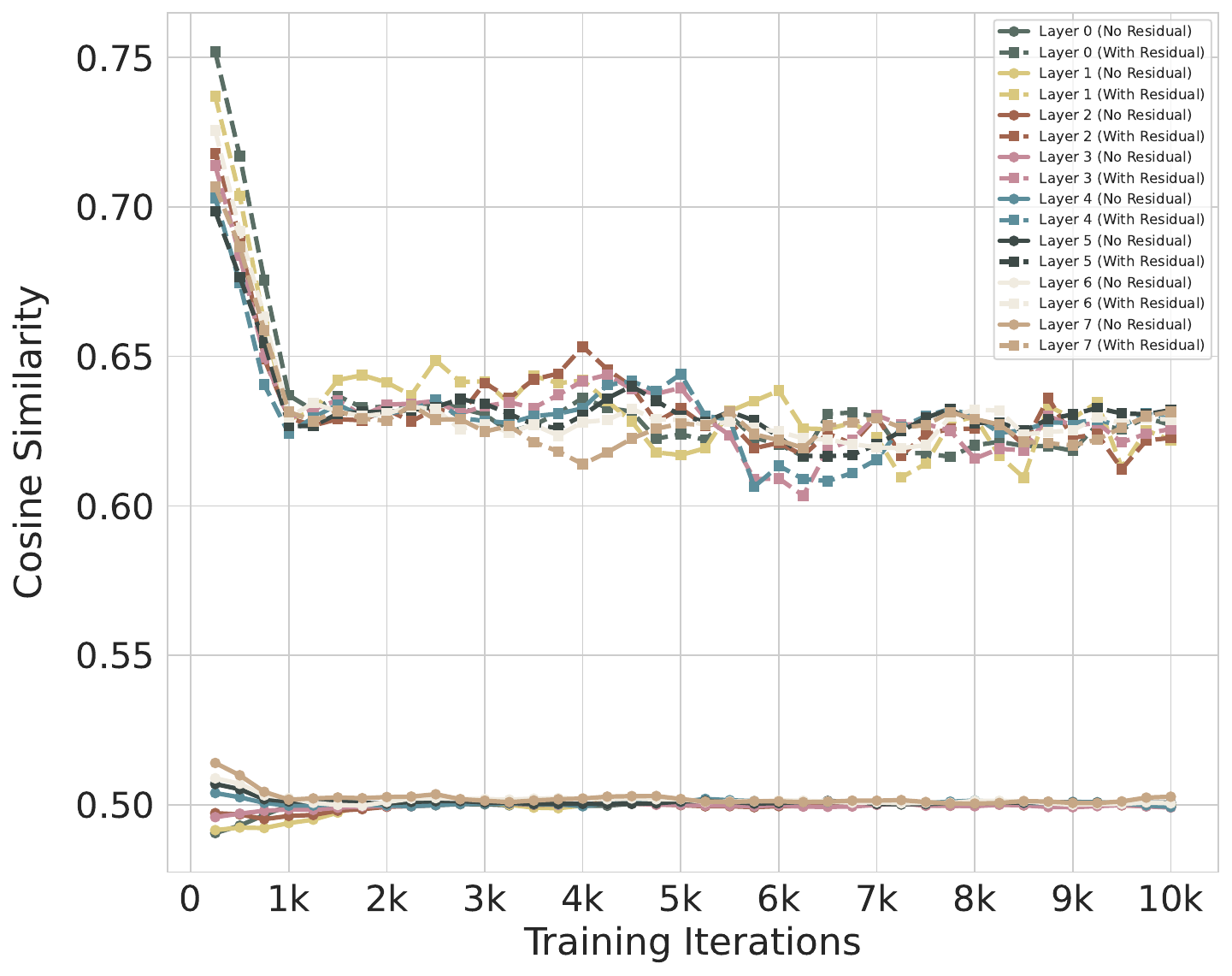}
        \caption{\name $l=2$}
    \end{subfigure}
    \hfill
    \begin{subfigure}[b]{0.49\linewidth}
        \centering
        \includegraphics[width=\linewidth, height=0.7\textwidth]{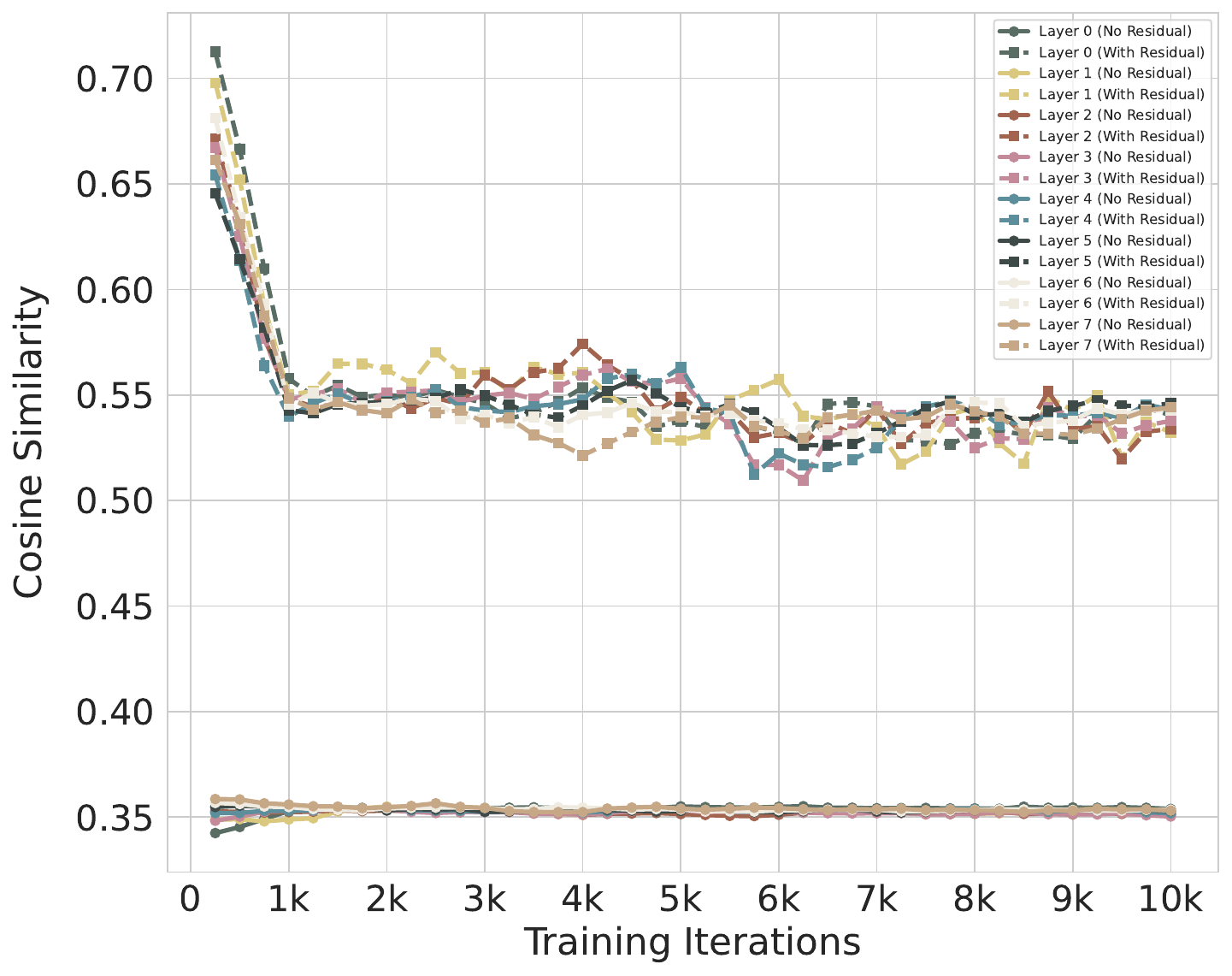}
        \caption{\name $l=3$}
    \end{subfigure}
    \caption{\textbf{Cosine Similarities between the Update Matrices of \name with or without Residual and Adam.} We report the average similarity across all modules within each layer. As observed, the update matrices including the residual term exhibit a higher cosine similarity with Adam’s updates compared to those without the residual. Specifically, for the setting $l=3$, \name updates maintain a cosine similarity greater than $0.5$ with standard Adam, despite retaining only $1/8$ of the original Adam optimizer state.}
    \label{fig:cosine_similarity}
\end{figure*}

\begin{figure*}[!th]
    \centering
    \begin{subfigure}[b]{0.48\linewidth}
        \centering
        \includegraphics[width=\linewidth, height=0.7\textwidth]{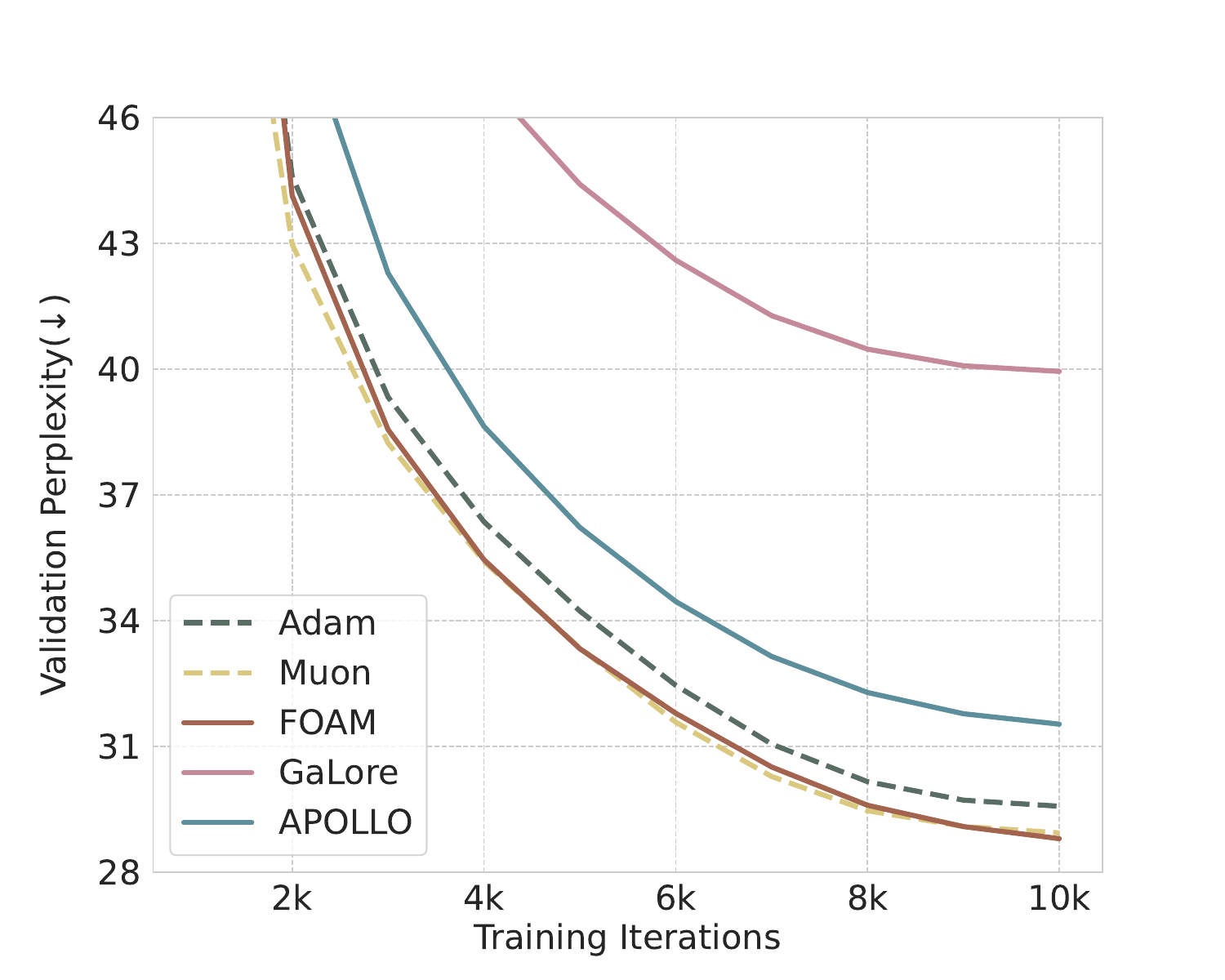}
        \caption{Pre-training LLaMA-60M}
    \end{subfigure}
    \hfill
    \begin{subfigure}[b]{0.48\linewidth}
        \centering
        \includegraphics[width=\linewidth, height=0.7\textwidth]{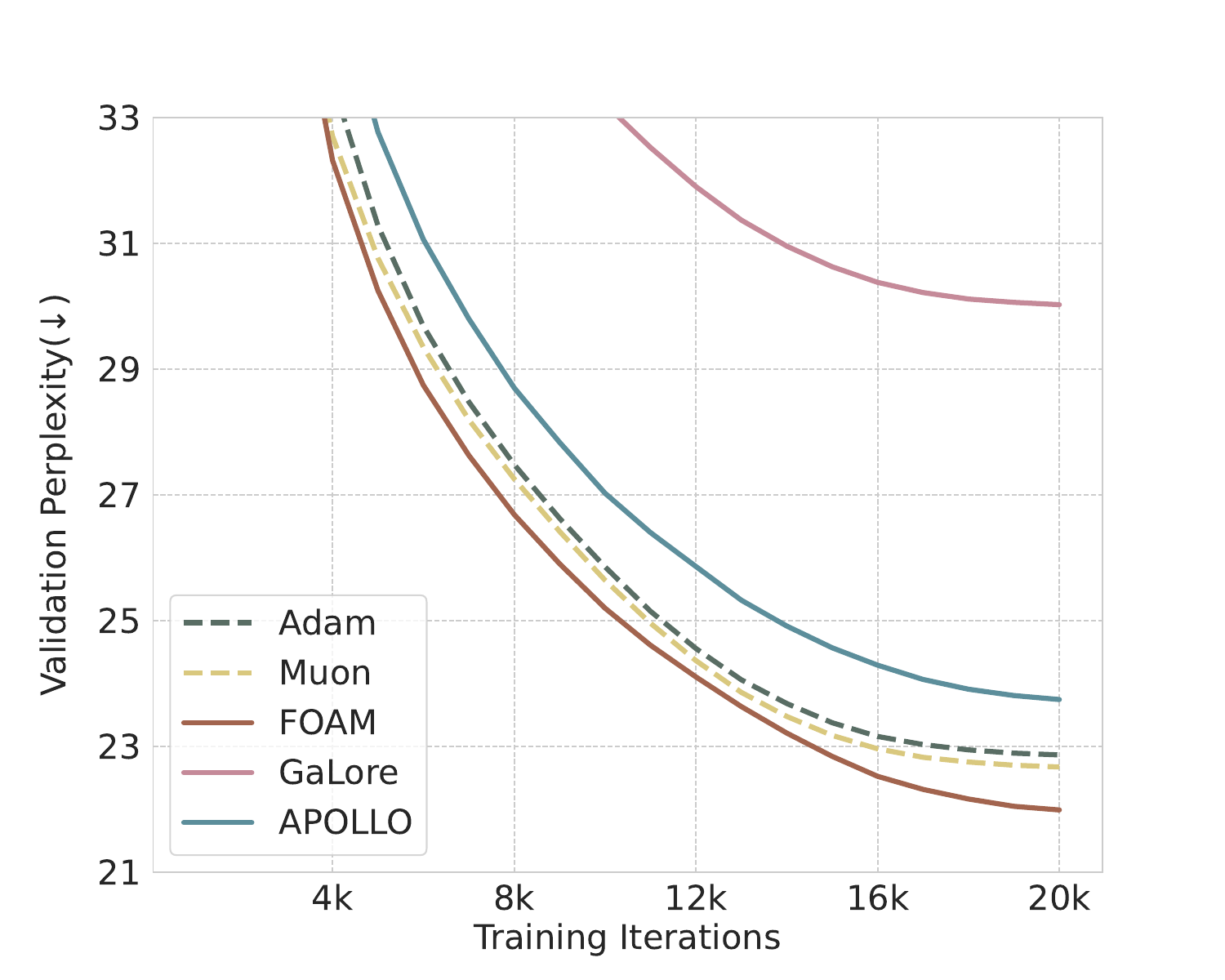}
        \caption{Pre-training LLaMA-130M}
    \end{subfigure}
    \caption{PPL learning curves of pre-training LLaMA-60M and 130M on C4}
    \label{fig:ppl_curve_60m_130m}
\end{figure*}

\subsection{GPT-2 and DeBERTa Experiments}
In this section, we test the performance of the \name we proposed on GPT-2 \cite{radford2019gpt2}, and DeBERTa \cite{he2021debertadecodingenhancedbertdisentangled} models. For each model, we adjust the learning rate within the range \{5e-4, 1e-3, 2.5e-3, 5e-3, 1e-2\}, keeping the memory-efficient scaling factors unchanged, and train for a total of 20k iterations, covering 2.6B tokens. The experimental results are shown in Figure~\ref{fig:loss_deberta_gpt}. \name continues to achieve leading performance on these models, which demonstrates the scalability of \name across other model architectures.

\begin{figure*}[!th]
    \centering
    \begin{subfigure}[b]{0.48\linewidth}
        \centering
        \includegraphics[width=\linewidth, height=0.7\textwidth]{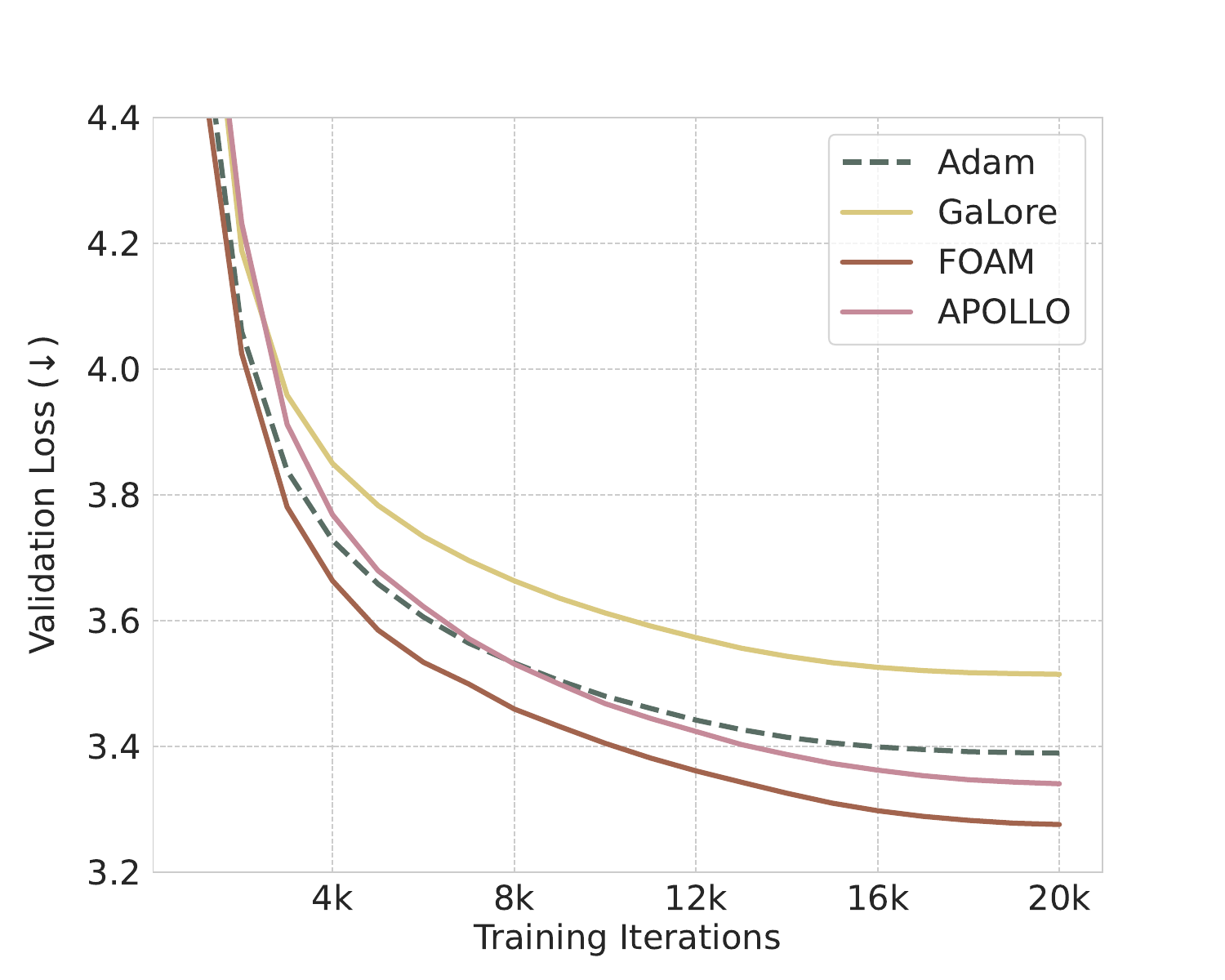}
        \caption{GPT-2}
    \end{subfigure}
    \hfill
    \begin{subfigure}[b]{0.48\linewidth}
        \centering
        \includegraphics[width=\linewidth, height=0.7\textwidth]{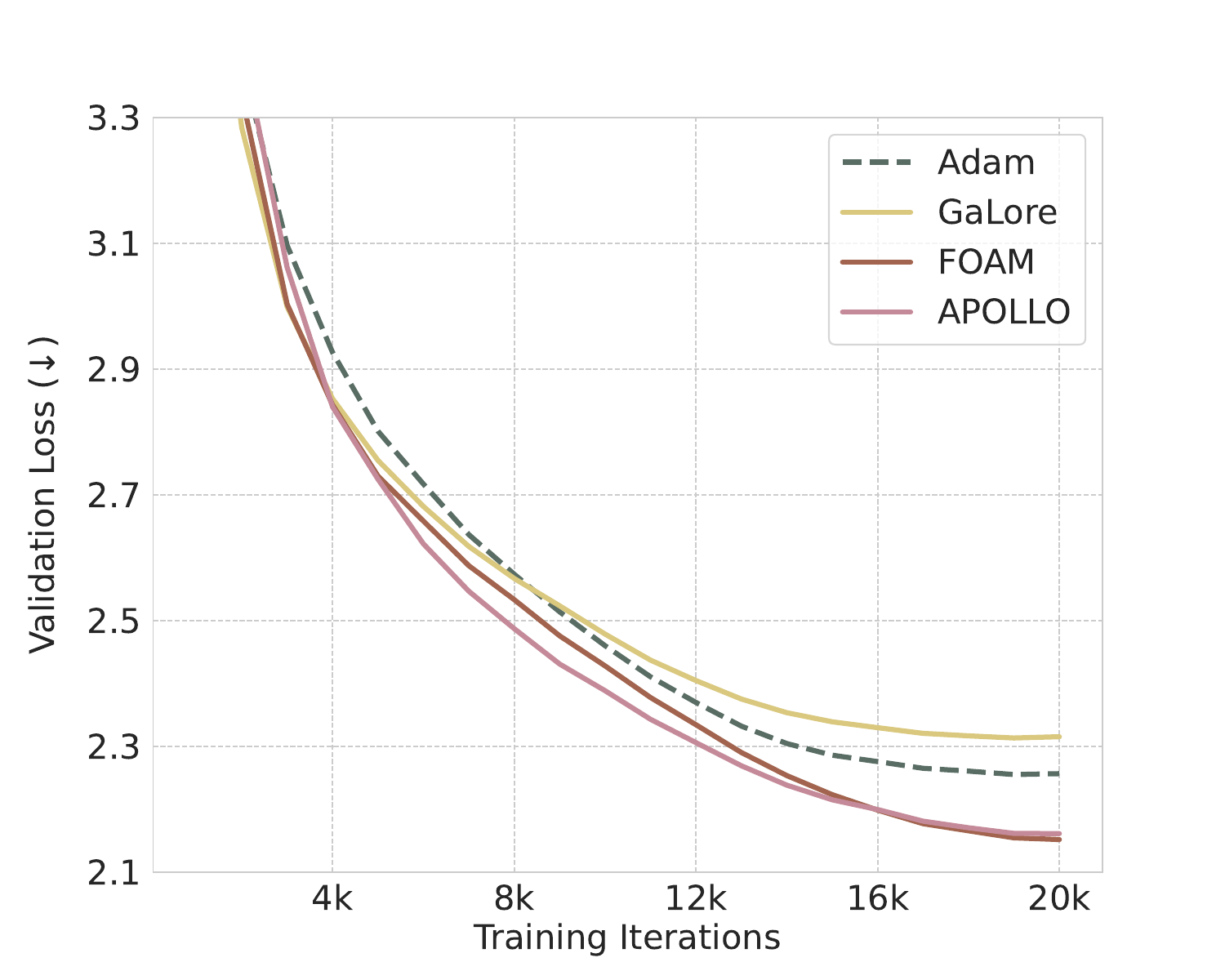}
        \caption{DeBERTa}
    \end{subfigure}
    \caption{\textbf{Pre-training GPT-2 and DeBERTa models on C4.} \textit{(a)}: Pre-training GPT-2-base model. \textit{(b)}: Pre-training DeBERTa-base model. \name continues to achieve leading performance on these models.}
    \label{fig:loss_deberta_gpt}
\end{figure*}


        
        
        
\subsection{Adam with Module-wise Learning Rate}
Notably, in this paper, \name surpasses Adam on validation metrics across many tasks. Similarly, Fira \cite{Chen2024FiraCW} and APOLLO \cite{zhu2024apollosgdlikememoryadamwlevel_apollo} also outperform Adam, even though these memory-efficient methods follow the same hyperparameter tuning strategy as Adam. In this section, we attempt to provide a possible explanation for this phenomenon, namely that it stems from the module-wise optimizer configuration used in current memory-efficient optimizers.

Concretely, most current memory-efficient optimizers \cite{zhao2024galore,jordan2024muon,Chen2024FiraCW,zhu2024apollosgdlikememoryadamwlevel_apollo} use a hybrid optimizer setup—employing vanilla Adam for modules like Embeddings and LayerNorm, while applying compressed-state optimization to Attention and MLP modules, with a scaling factor $\alpha$ used to adjust the learning rates across modules. For modules such as Embeddings and LayerNorm, the learning rate $lr$ is applied, while Attention and MLP modules use $lr \times \alpha$, effectively creating a module-dependent learning-rate scheme. Prior work \cite{wang2025sharpness} has shown that Adam with module-wise learning rates converges faster than standard Adam.

Accordingly, to thoroughly assess the \name optimizer, we begin with an ablation study on the scaling factor $\alpha$; the experimental results are presented in Figure~\ref{fig:ablation_alpha}. The results indicate that \name is in fact not sensitive to this hyperparameter. In all of our pre-training and fine-tuning experiments, FOAM consistently uses a scaling factor $\alpha=0.25$, whereas GaLore requires task-specific tuning. This further demonstrates FOAM’s robustness to this parameter and the general applicability of this setting.

Furthermore, we compare \name with an Adam variant that employs the same block-wise learning-rate scheme, with results presented in Figure~\ref{fig:adam_alpha}. The block-wise learning-rate configuration substantially improves Adam’s performance. Under this setting, \name achieves results comparable to Adam, further demonstrating its effectiveness: with identical hyperparameters, \name performs on par with a well-tuned Adam while requiring less memory and achieving higher training throughput.

\begin{figure*}[!th]
    \centering
    \begin{subfigure}[b]{0.48\linewidth}
        \centering
        \includegraphics[width=\linewidth, height=0.7\textwidth]{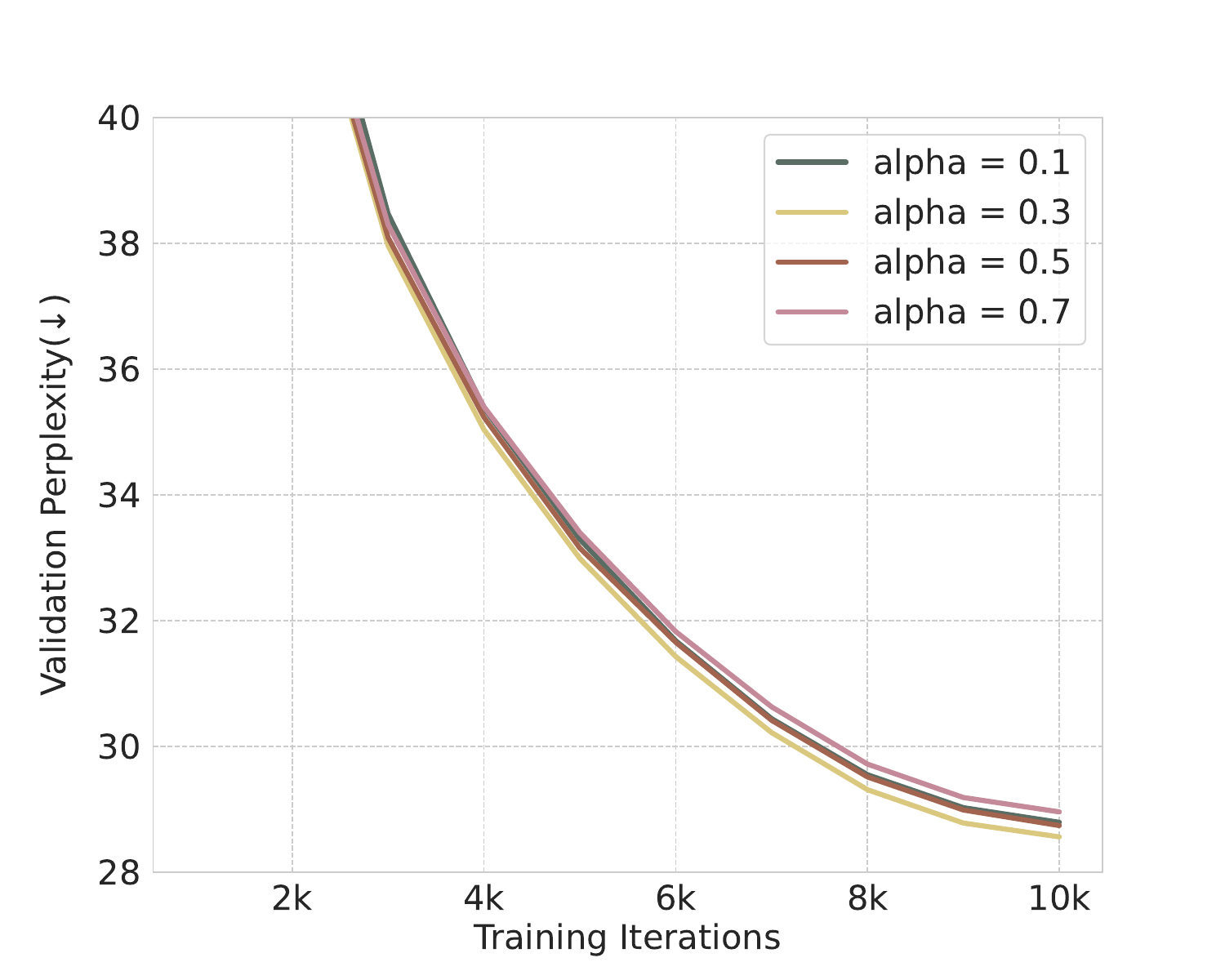}
        \caption{60M model pre-training}
    \end{subfigure}
    \hfill
    \begin{subfigure}[b]{0.48\linewidth}
        \centering
        \includegraphics[width=\linewidth, height=0.7\textwidth]{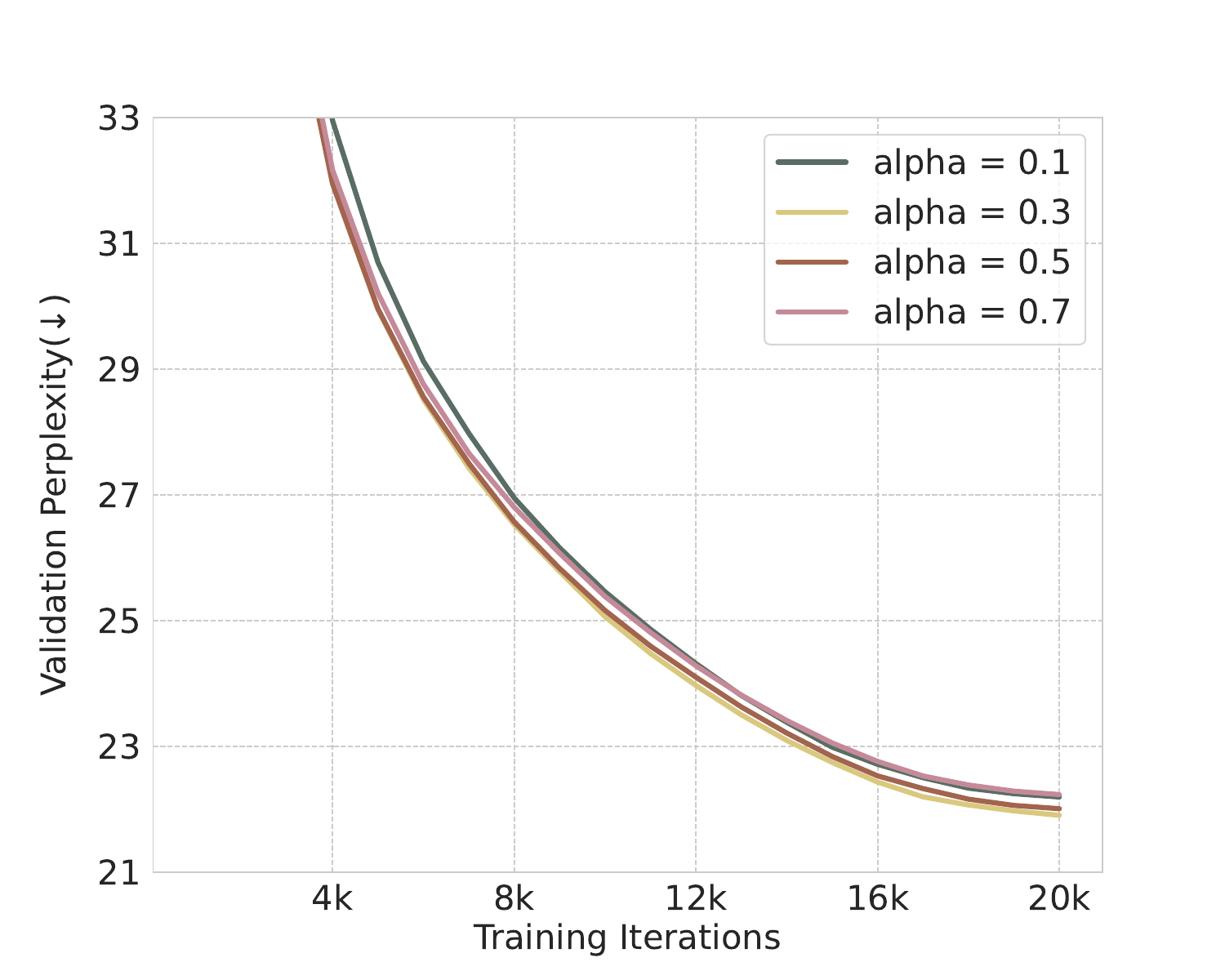}
        \caption{130M model pre-training}
    \end{subfigure}
    \caption{\textbf{Study the effects of $\alpha$ in Algorithm \ref{algo:fold_adam_algo}.} As observed, \name is not sensitive to the choice of $\alpha$ in our tests.}
    \label{fig:ablation_alpha}
\end{figure*}

\begin{figure*}[!th]
    \centering
    \begin{subfigure}[b]{0.48\linewidth}
        \centering
        \includegraphics[width=\linewidth, height=0.7\textwidth]{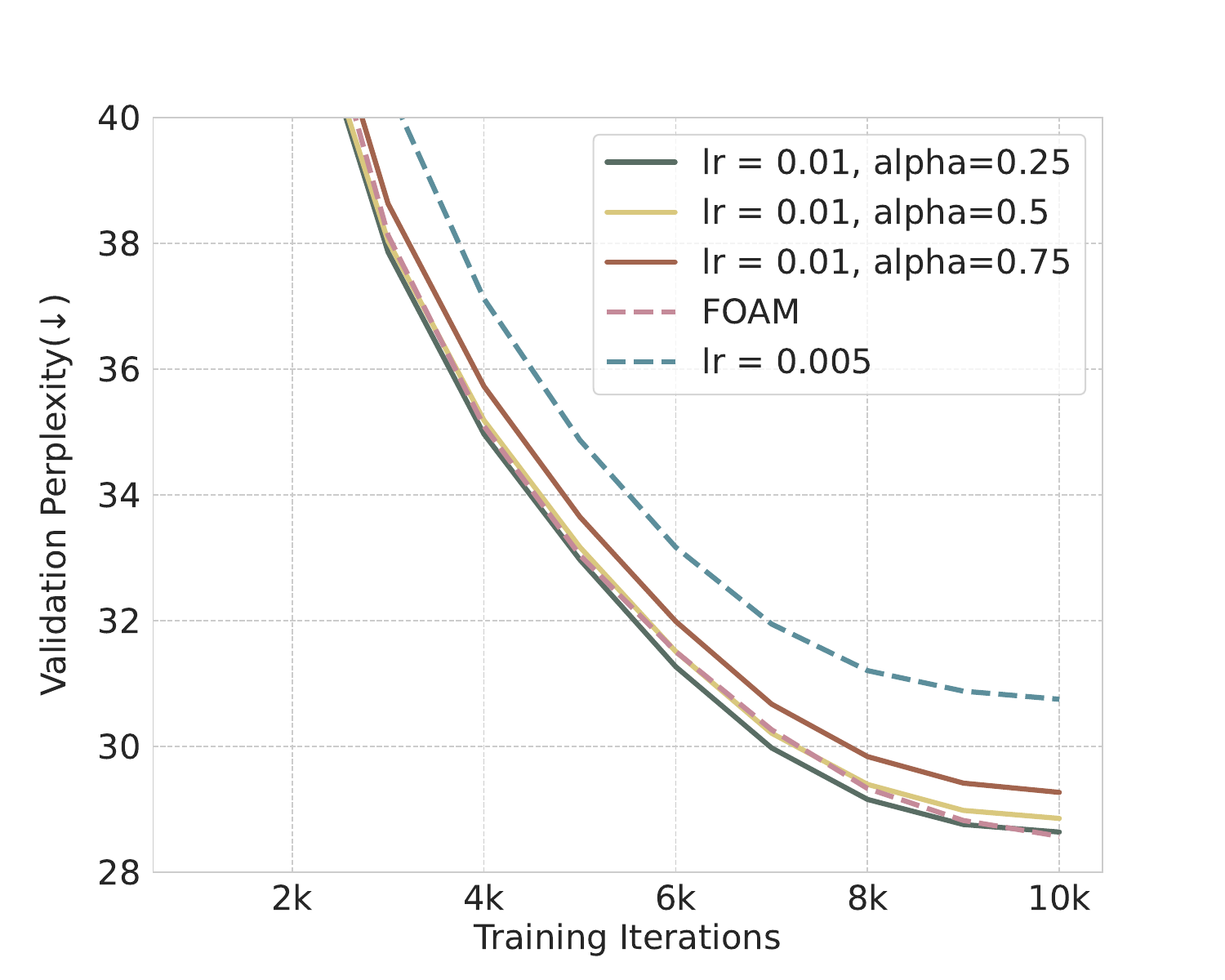}
        \caption{60M model pre-training}
    \end{subfigure}
    \hfill
    \begin{subfigure}[b]{0.48\linewidth}
        \centering
        \includegraphics[width=\linewidth, height=0.7\textwidth]{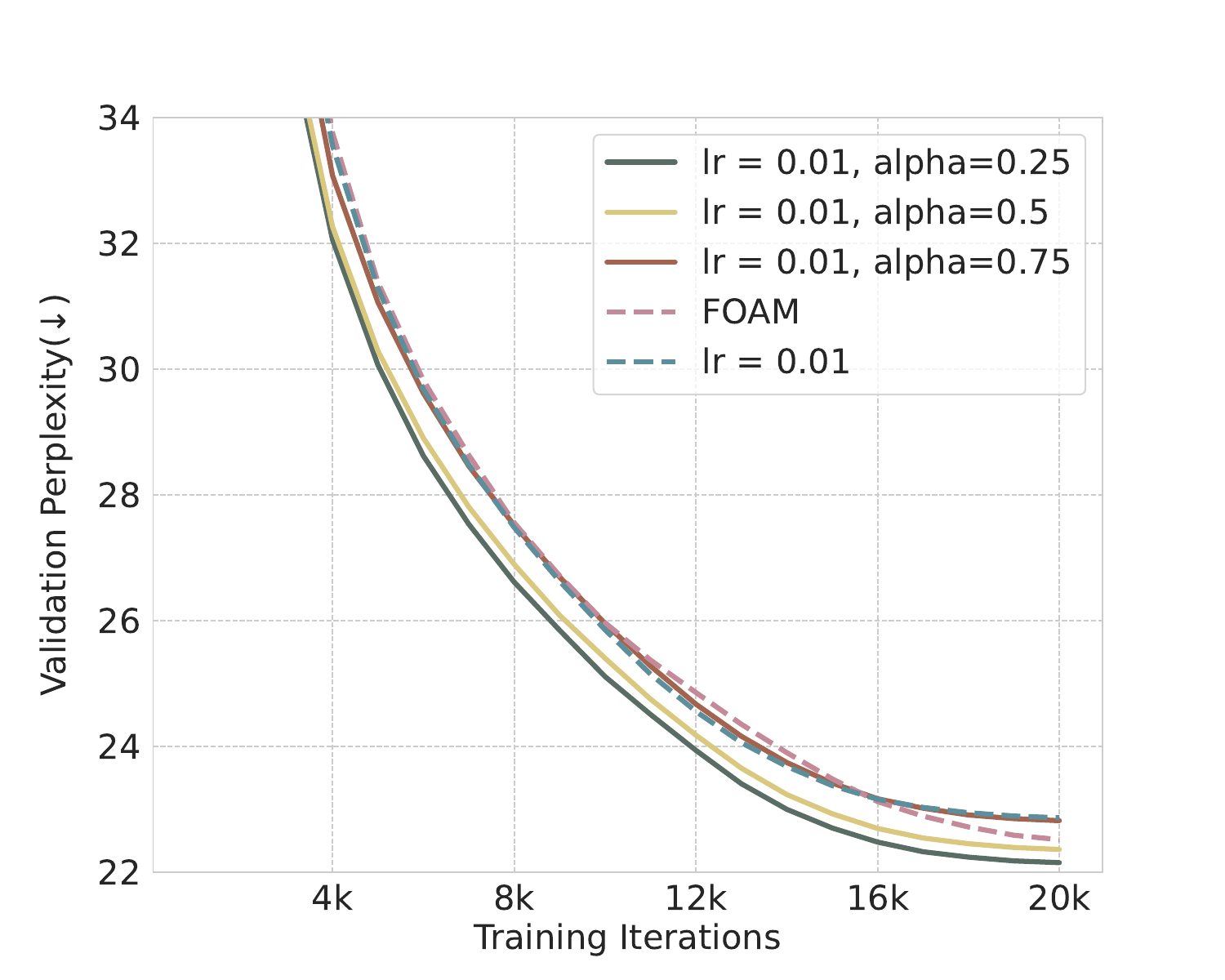}
        \caption{130M model pre-training}
    \end{subfigure}
    \caption{\textbf{Comparing FOAM with Adam employing block-wise learning rates.} Block-wise Adam converges more rapidly and yields a lower final validation loss than the uniform-rate variant. \name’s performance remains on par with that of block-wise Adam.}
    \label{fig:adam_alpha}
\end{figure*}

\section{Future Works}
Several directions remain open for further exploration:
\begin{itemize}
    \item Due to limitations in computational resources, the maximum model size we used to validate the effectiveness of \name was 7B, while experiments on larger models often necessitate more than 100 high-memory GPUs. Thus, the full potential of FOAM for training ultra-large-scale models on significantly larger GPU clusters remains to be elucidated.
    \item The potential of extending \name to encompass additional model architectures, including diffusion models \cite{song2021scorebased} and Vision Transformers (ViT) \cite{dosovitskiy2020vit}. The optimization of such models frequently necessitates the handling of numerous gradients exceeding two dimensions. The adaptation of \name to effectively handle these high-dimensional gradients necessitates further architectural design and is thus left for future work.
    \item Investigating how the \name compression strategy influences the overhead of inter-GPU communication. Notably, \name only performs compression on neighboring elements, thereby eliminating the need for pre-storing the complete gradient. Consequently, \name demonstrates natural compatibility with gradient segmentation methods like ZeRO \cite{rajbhandari2020zero}. Since the minimization of inter-GPU communication overhead falls outside the core scope of this manuscript, we reserve this analysis for subsequent research. 
\end{itemize}

\begin{figure*}[!ht]
    \centering
    \begin{subfigure}[b]{0.33\linewidth}
        \centering
        \includegraphics[width=\linewidth, height=0.65\textwidth]{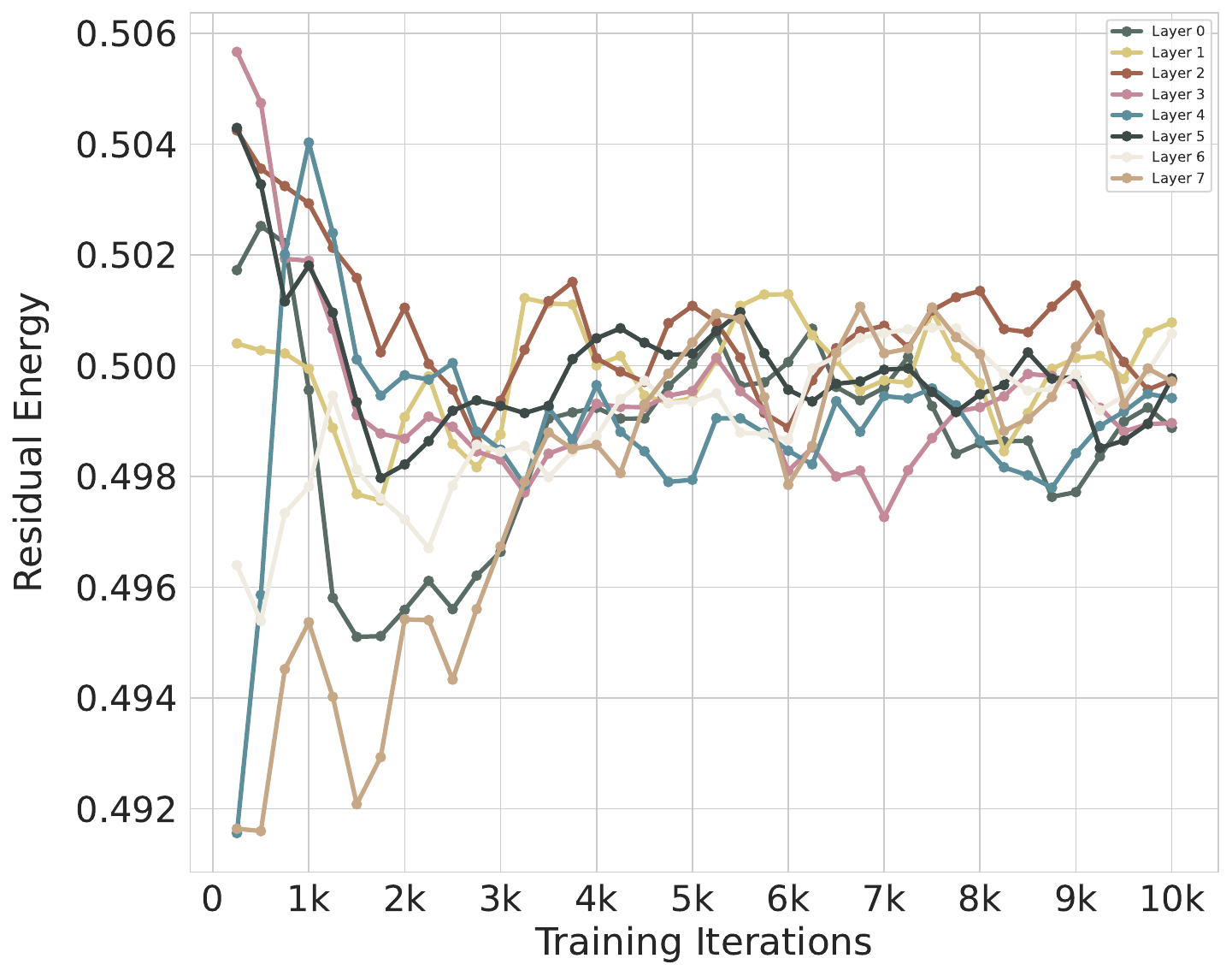}
        \caption{\name $l=1$}
    \end{subfigure}
    \hfill
    \begin{subfigure}[b]{0.33\linewidth}
        \centering
        \includegraphics[width=\linewidth, height=0.65\textwidth]{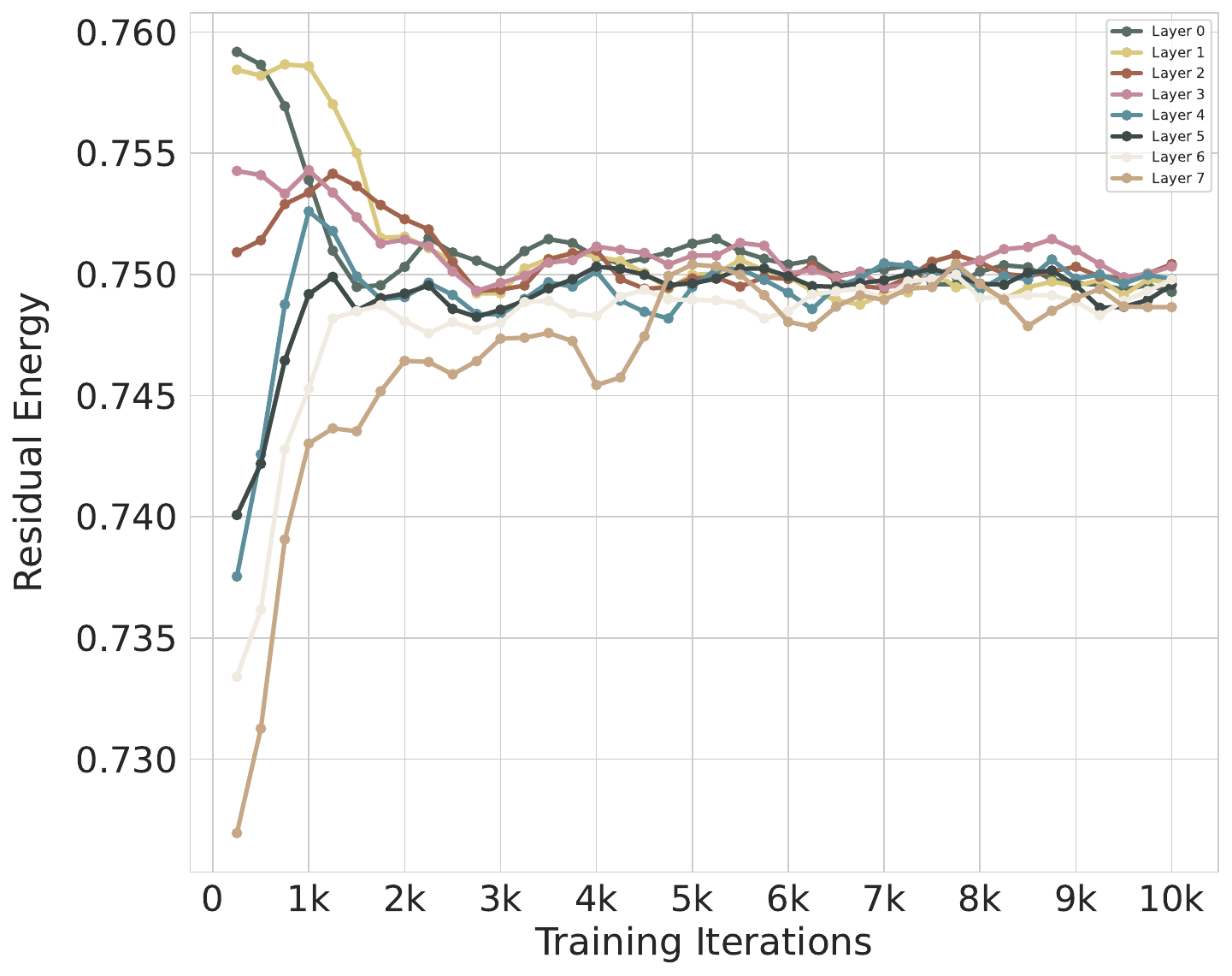}
        \caption{\name $l=2$}
    \end{subfigure}
    \hfill
    \begin{subfigure}[b]{0.33\linewidth}
        \centering
        \includegraphics[width=\linewidth, height=0.65\textwidth]{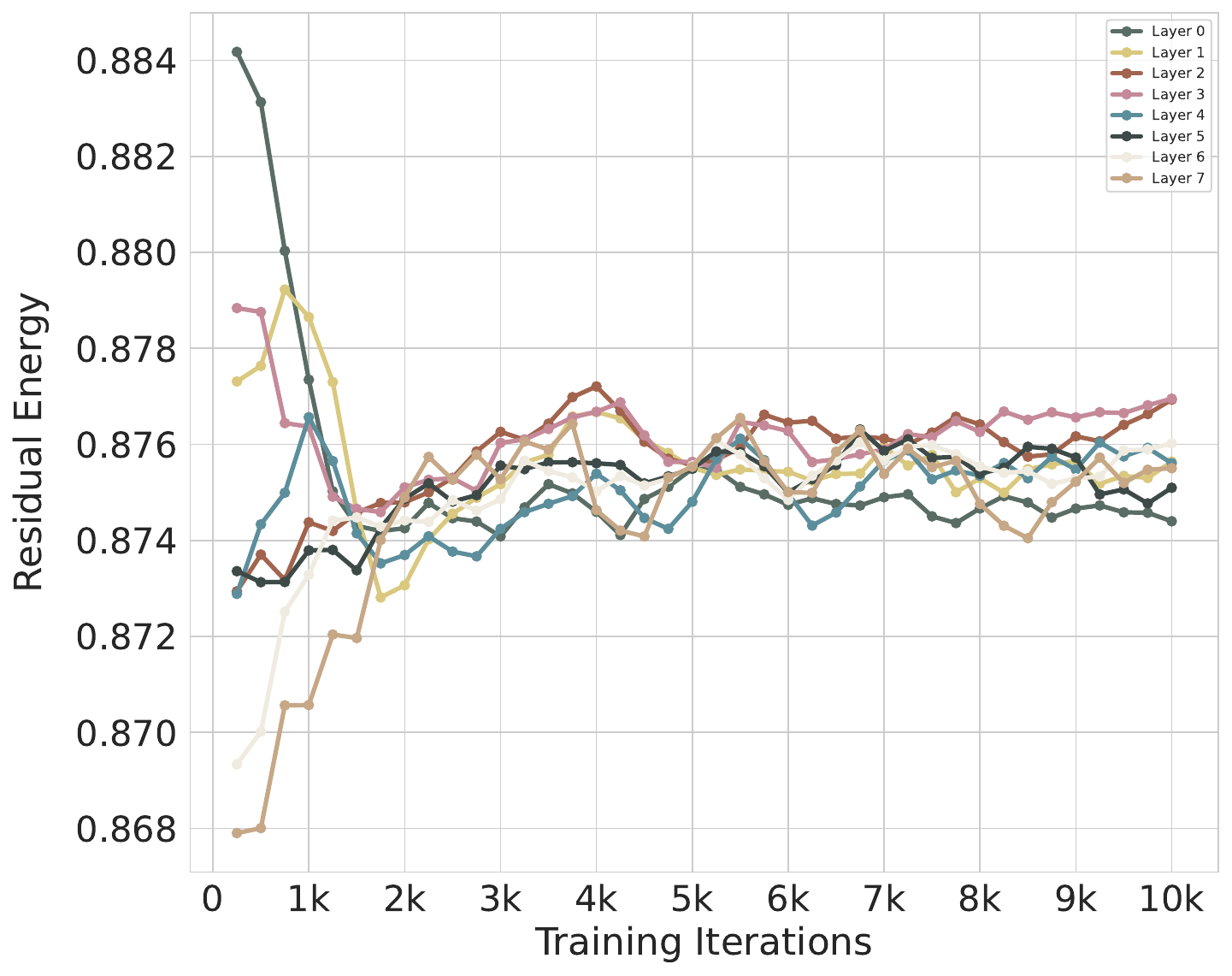}
        \caption{\name $l=3$}
    \end{subfigure}
    \caption{\textbf{The variation of residual energy ratio throughout training with different fold level $l$}. We report the average energy ratio across all modules within each layer.It can be observed that the energy ratio of the residual increases as $l$ grows, with the values concentrating around $1- \frac{1}{2^l}$. This implies that the residual $R_{t}$ captures most of the energy from the original gradient, highlighting the necessity of injecting residuals.} 
    \label{fig:residual_energy}
\end{figure*}

\begin{table*}[!ht]
    \centering
    \setlength{\tabcolsep}{15pt}
    \renewcommand{\arraystretch}{1.1}
    \caption{Recorded memory overhead for different methods on LLaMA-1.3B with a batch size of 32. We train the LLaMA-1.3B model on 32 NVIDIA RTX 3090 (24GB) GPUs.}
    \label{tab:throught_1B}
    \begin{tabular}{l|c|l|c}
    \toprule
        Methods & Memory & Methods & Memory \\
        \midrule
        Full-Adam & 20.61G & GaLore-1/4 & 17.55G \\
        Muon & 18.30G & APOLLO-1/4 & 17.56G \\
        Adam-mini & 18.05G & \cellcolor{blue4}\name-2 & \cellcolor{blue4}17.25G \\
        \midrule
        GaLore-1/8 & 16.74G & GWT-Mini & 16.00G \\
        APOLLO-1/8 & 16.74G & APOLLO-Mini & 16.00G \\
        \cellcolor{blue4}\name-3 & \cellcolor{blue4}16.57G & \cellcolor{blue2}\name-Mini & \cellcolor{blue2}16.00G \\
         \bottomrule
    \end{tabular}
\end{table*}


\begin{figure*}[!th]
    \centering
    \begin{subfigure}[b]{0.48\linewidth}
        \centering
        \includegraphics[width=\linewidth, height=0.75\textwidth]{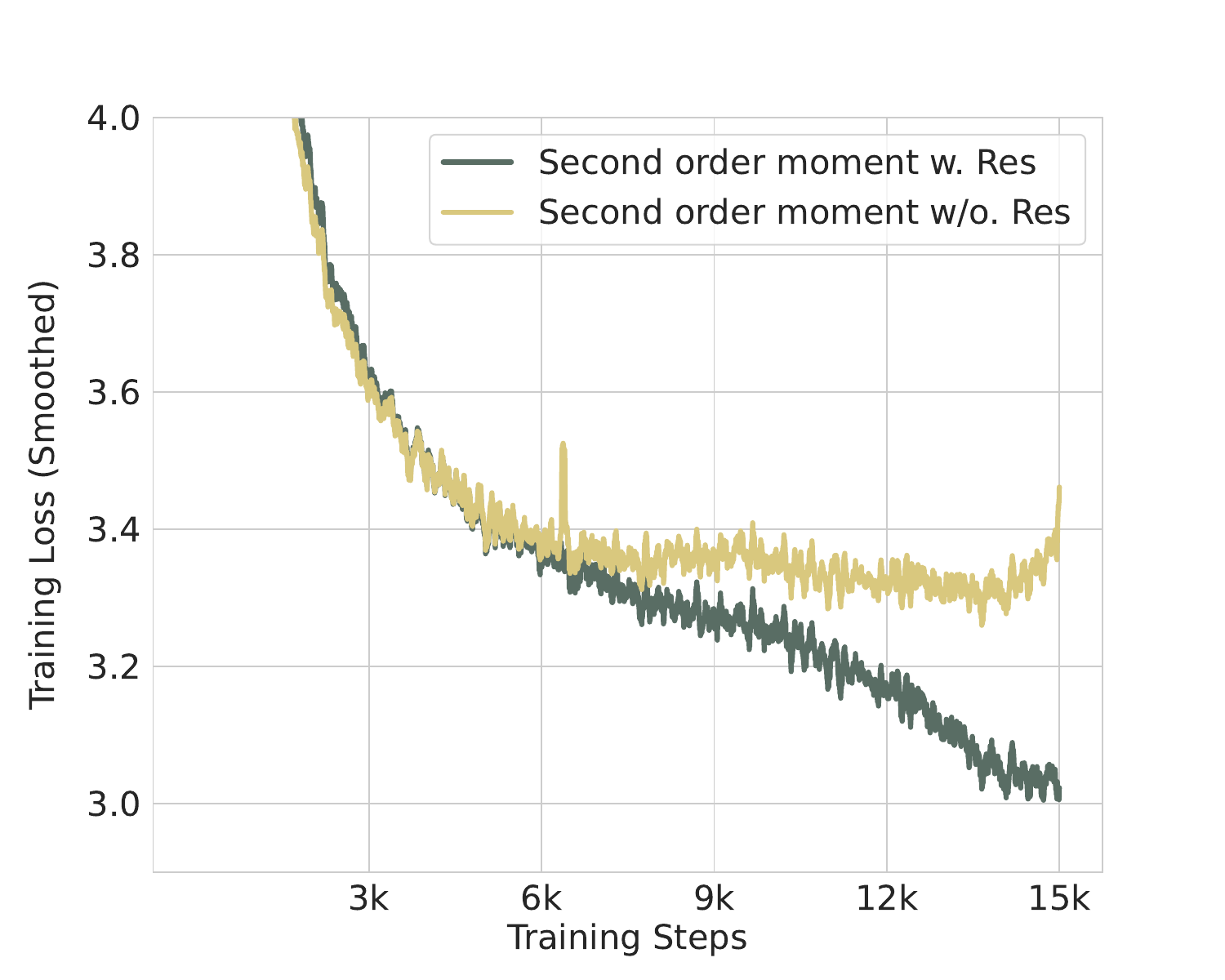}
        \caption{Training Loss on LLaMA-3B}
    \end{subfigure}
    \hfill
    \begin{subfigure}[b]{0.48\linewidth}
        \centering
        \includegraphics[width=\linewidth, height=0.75\textwidth]{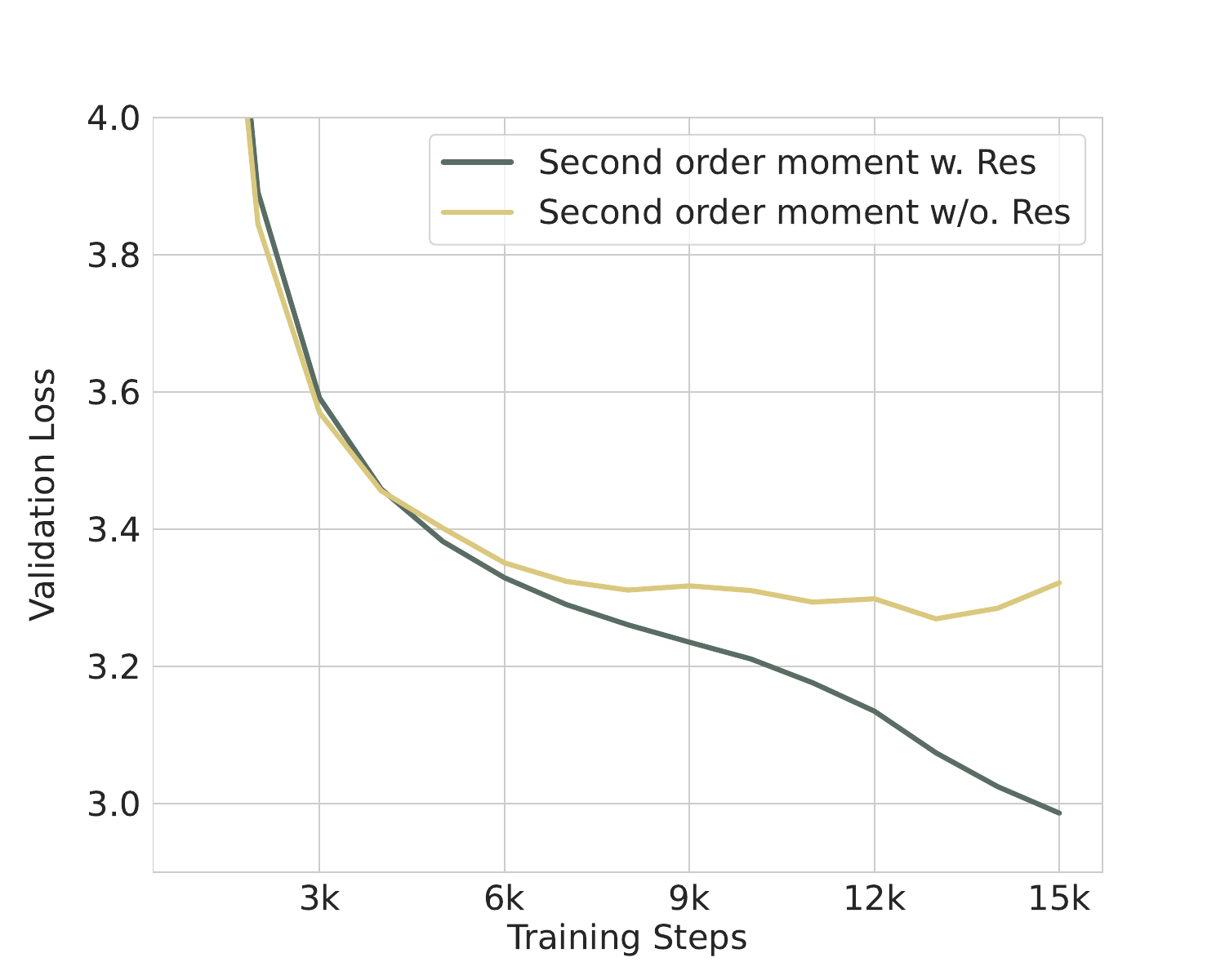}
        \caption{Validation Loss on LLaMA-3B}
    \end{subfigure}
    \caption{\textbf{Ablation study of $R_{t}^2$ on the second-order moment.} The results show that incorporating residuals into the second-order moment leads to a more stable decrease in the training curve. In contrast, without residuals, the training exhibits a faster initial decrease but experiences a rise in loss during later stages.}
    \label{fig:ablation_3b_res}
\end{figure*}

\section{Benchmark Details}
In this work, we evaluate our methods using several widely adopted benchmark datasets that cover both pre-training and downstream tasks.

\begin{itemize}
    \item \textbf{C4 (Colossal Clean Crawled Corpus):}
    
    C4 is a large-scale English-language text corpus derived from Common Crawl data. It has been widely used for language model pretraining due to its scale and linguistic diversity. We follow the preprocessing and filtering steps introduced by the T5 framework to remove boilerplate and low-quality content.

    \item \textbf{MMLU (Massive Multitask Language Understanding):}
    
    MMLU is a comprehensive benchmark covering 57 tasks across various domains, designed to evaluate the reasoning and world knowledge of language models in a zero-shot setting. These tasks span multiple categories, including: STEM (e.g., physics, chemistry, mathematics), Humanities (e.g., history, philosophy, art), Social sciences (e.g., economics, psychology, political science), Other professional and academic subjects (e.g., law, computer science, clinical knowledge)

    \item \textbf{GLUE (General Language Understanding Evaluation):}
    
    GLUE consists of nine NLU tasks that test a model’s general language understanding capabilities: CoLA (linguistic acceptability) \cite{Warstadt2018NeuralNAcola}, STS-B (semantic textual similarity) \cite{Cer2017SemEval2017T1sts-b}, MRPC (paraphrase detection) \cite{Dolan2005AutomaticallyCAmrpc}, RTE (recognizing textual entailment), SST-2 (sentiment analysis) \cite{Socher2013RecursiveDMsst-2}, MNLI (multi-genre natural language inference) \cite{Williams2017ABCmnli}, QNLI (question-answering NLI) \cite{Rajpurkar2018KnowWYqnli}, QQP (duplicate question detection). The broad coverage makes GLUE a standard benchmark for evaluating pre-training BERT models \cite{hu2021lora}.
\end{itemize}


\end{document}